\documentclass[11pt]{article}

\usepackage[a4paper,margin=1in]{geometry}
\usepackage{amsmath,amssymb,amsthm,mathtools}
\usepackage{graphicx}
\usepackage{subcaption}
\usepackage{enumitem}
\usepackage{algorithm}
\usepackage{algpseudocode}
\usepackage[authoryear]{natbib}
\usepackage[colorlinks=true,
            citecolor=blue,
            linkcolor=blue,
            urlcolor=blue]{hyperref}

\usepackage[authoryear]{natbib}

\usepackage{algorithm}
\usepackage{algpseudocode}
\newcommand{\WrapState}[1]{\State \parbox[t]{0.88\linewidth}{#1}}

\usepackage{booktabs}   % for \toprule, \midrule, \bottomrule
\usepackage{multirow}   % for \multirow
\usepackage{graphicx}   % for \resizebox

\usepackage{float}      % for [H]
\usepackage{subcaption} % for subfigure
\usepackage{enumitem}

\newcommand{\ola}{\overleftarrow}
\newcommand{\ora}{\overrightarrow}

\newcommand{\rmd}{\mathrm{d}}
\newcommand{\rme}{\mathrm{e}}
\newcommand{\eqsp}{\,}

\newcommand{\eqdef}{:=}

\newcommand{\R}{\mathbb{R}}

\newcommand{\E}{\mathbb{E}}

\newcommand{\Id}{\mathbf{I}}

\newcommand{\score}{\mathsf{s}}

\newcommand{\indep}{\perp\!\!\!\perp}
\newcommand{\dist}{\mathsf{d}}

\theoremstyle{plain}
\newtheorem{theorem}{Theorem}[section]
\newtheorem{proposition}[theorem]{Proposition}
\newtheorem{lemma}[theorem]{Lemma}

\theoremstyle{definition}

\theoremstyle{remark}
\newtheorem{remark}[theorem]{Remark}

\newtheorem*{proposition*}{Proposition}

\title{A Stochastic Analysis Perspective on Diffusion Models}
\author{
Stanislas Strasman\\
\small Laboratoire de Probabilités, Statistique et Modélisation\\
\small Sorbonne Université
}
\date{}

\begin{document}

\maketitle

\begin{abstract}
Generative modelling involves two closely related tasks: \emph{learning}
a probabilistic approximation of an unknown data-generating distribution
from limited observations, and designing a tractable mechanism for
\emph{sampling} from this approximation. We present
a synthesis of ideas from sampling theory, score estimation,
and diffusion time reversal, using stochastic calculus as a common
framework for understanding and analysing Score-based Generative Models (SGMs).

We first consider learning and sampling separately, showing how score
functions provide a natural bridge between these two tasks. The
computational convenience and regularising properties of Gaussian
smoothing motivate denoising score matching, which estimates the
marginal scores of noisy data distributions by regression against
tractable conditional scores. The trade-off between smoothing bias,
score-estimation difficulty, and mixing behaviour then motivates the
use of multiple noise levels and annealed Langevin dynamics.

We next highlight the distinction between constructing dynamics that
preserve a fixed target distribution and constructing dynamics whose
marginals follow a prescribed evolution. This motivates the diffusion
time-reversal viewpoint and the passage from marginal distributions
to probability laws on path space. In particular, matching one-time
marginal distributions alone does not establish that a proposed
dynamics is the time reversal of a given forward diffusion. A
martingale-based proof sketch makes this path-space statement explicit
and explains the score-dependent correction to the reverse drift.

Finally, we connect the ideal reverse diffusion to implementable
training and sampling procedures and provide an entry point to the
literature on stability and convergence of Score-based Generative
Models. We explain how practical approximations introduce
initialization, score approximation, and time-discretization errors,
and discuss the probability metrics and divergences used to analyse
their impact on the generated distribution. We also distinguish these
theoretical criteria from empirical metrics used to evaluate generated
samples.
\end{abstract}

\tableofcontents

\newpage

\section{Introduction}
\label{sec:introduction}

Generative modelling aims to approximate the probability distribution
underlying observed data and to provide a mechanism for generating new
samples with similar statistical properties. We assume access to a finite collection of observations
\[
X_1,\ldots,X_n
\stackrel{\mathrm{i.i.d.}}{\sim}
\pi_{\rm data},
\]
where $\pi_{\rm data}$ is an unknown probability distribution on
$\mathbb R^d$. The objective is to construct a model distribution
$\widehat\pi$ that approximates $\pi_{\rm data}$, in a sense to be
specified, and from which samples can be generated at a reasonable
computational cost. This involves two complementary tasks:
\emph{learning} a probabilistic description from finite observations,
and \emph{sampling} from the resulting model.

Diffusion-based generative methods approach these tasks by progressively
corrupting data with noise and learning how to reverse this corruption
\citep{dicksteinthermo,song2019generative,ho2020denoising,song2021scorebased}.
Figure~\ref{fig:diffusion_philosophy_intro} illustrates this construction.

\begin{figure}[tbp]
    \centering
    \includegraphics[width=0.82\textwidth]
        {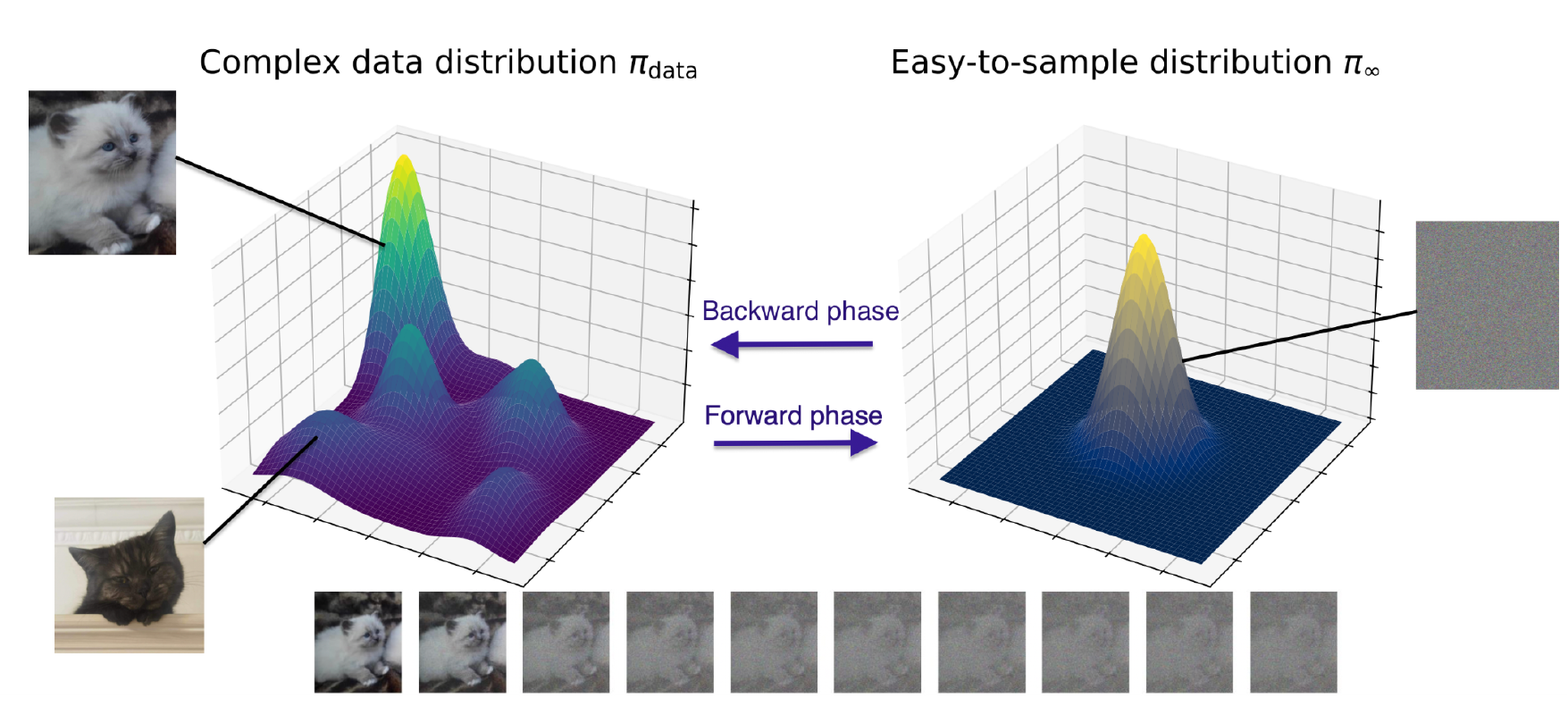}
    \caption{Forward noising and reverse generation. During the forward phase, data are progressively corrupted with noise, bringing their distribution close to a simple reference law. Starting from this reference distribution, the generative phase progressively denoises samples, producing a distribution that approximates the data distribution.
    }
    \label{fig:diffusion_philosophy_intro}
\end{figure}

Diffusion models can be motivated by comparison with other generative
approaches. Their standard denoising score-matching formulation does
not require the model-dependent MCMC often used in likelihood-based
training of energy-based models
\citep{ackley1985learning,hinton2002training},
the learned variational encoder of VAEs
\citep{kingma2014autoencoding,rezende2014stochastic},
the adversarial optimisation of GANs
\citep{goodfellow2014generative}, the coordinate ordering of
autoregressive models \citep{bengio2003neural}, or the explicitly
invertible neural parametrisations of normalising flows
\citep{dinh2017density,kingma2018glow}.

These comparisons clarify differences in training, sampling, and model
design, but do not by themselves explain the underlying probabilistic
construction. Here, we instead emphasize the connection with classical
sampling theory through Langevin dynamics and with score matching,
which constitute two foundational ingredients of score-based generative
models \citep{hyvarinen2005estimation,Vincent,song2019generative}.

\section{Learning and sampling with score functions}

\subsection{A sampling mechanism: Langevin dynamics}

\paragraph{The Langevin equation} 

To build intuition, let us momentarily assume that the data distribution
$\pi_{\mathrm{data}}$ on $(\mathbb{R}^d,\mathcal{B}(\mathbb{R}^d))$ is known and admits a density $p_0$ with respect to the Lebesgue measure. Suppose that $p_0$ can be written in Gibbs form as
$$
p_0(x) \eqdef \frac{1}{Z}\exp \left\{-U(x) \right\} \eqsp,
\qquad
Z \eqdef \int_{\mathbb{R}^d}\exp\!\big(-U(x)\big) \rmd x \in (0,\infty) \eqsp.
$$
For simplicity, we assume throughout this discussion that $U\in C^2(\mathbb{R}^d)$.
In this idealized setting, the generative modeling task reduces to designing
a procedure to sample from $\pi_{\mathrm{data}}$.

A classical approach is to construct a Markov diffusion process $(X_t)_{t\ge 0}$
having $\pi_{\mathrm{data}}$ as invariant distribution. A canonical example is the
\emph{Langevin diffusion} \citep{langevin1908brownian,doob1942brownian,nelson1967dynamical}:
\begin{align} \label{eq:langevin}
\rmd X_t
&= \nabla \log p_0(X_t)\,\rmd t + \sqrt{2}\,\rmd B_t 
\qquad X_{0}\sim \mu_0 \eqsp,
\end{align}
where $(B_t)_{t\ge 0}$ is a standard $d$-dimensional Brownian motion on a
filtered probability space satisfying the usual conditions and
$\mu_0\in \mathcal{P} (\mathbb R^d)$ is an arbitrary initial distribution. 

The convergence properties of \eqref{eq:langevin} to its stationary distribution $\pi_{\rm data}$ depend on the geometry of the
potential $U$, as well as on regularity and growth assumptions ensuring that
the dynamics is well defined and ergodic. The precise assumptions depend on the distance used to measure
convergence. Under suitable conditions, the Langevin diffusion converges exponentially fast to its invariant distribution in
$2$-Wasserstein \citep{villani_optimal_2009} or in total variation \citep{roberts1996exponential}.

A particularly transparent setting is the strongly log-concave case. We say that $U$ is $\alpha$-strongly convex \citep{saumard2014log}, if there exists $\alpha>0$ such that, for all $x,y\in\mathbb R^d$
\begin{equation} \label{eq:strong_convexity}
\left\langle \nabla U(x)-\nabla U(y), x-y \right\rangle \ge \alpha \| x-y \|^2 \eqsp.
\end{equation}

\begin{proposition}[Strongly convex case]\label{prop:w2_contract}
Assume that, $\mathcal{W}_2(\mu_0,\pi_{\rm data})<\infty$, that $U$ is $\alpha$-strongly convex and let $(X_t)_{t\ge 0}$ solve \eqref{eq:langevin}. Then $(X_t)_{t\ge 0}$ converges exponentially fast to $\pi_{\rm data}$ in the $2$-Wasserstein distance, for all $t \ge 0$
\begin{equation} \label{eq:wasserstein_contraction_langevin}
\mathcal{W}_2(\mu_t,\pi_{\rm data}) \le \rme^{-\alpha t} \mathcal{W}_2(\mu_0,\pi_{\rm data}) \eqsp,
\qquad \mu_t \eqdef \mathcal{L}(X_t) \eqsp.
\end{equation}
\end{proposition}

\begin{proof}[Sketch of proof]
Recall that a probability measure $\pi$ is invariant for a Markov semigroup $(P_t)_{t\ge 0}$ with
infinitesimal generator $L$ if $\pi L=0$  \citep[Chap. 4]{ethier1986markov} in the sense that
$$
\forall f\in \mathrm{Dom}(L),\qquad \int_{\mathbb{R}^d} (Lf)(x) \pi(\rmd x)=0 \eqsp.
$$
For the Langevin diffusion \eqref{eq:langevin}, the generator satisfies, for all
$f \in C_c^\infty(\mathbb{R}^d)$,
$$
Lf(x)= - \left\langle \nabla U(x), \nabla f(x) \right\rangle + \Delta f(x) \eqsp.
$$
By integration by parts, using that $\Delta f = \operatorname{div}(\nabla f)$,
$$
\int_{\mathbb{R}^d} \Delta f(x) p_0(x) \rmd x
= - \int_{\mathbb{R}^d} \left\langle  \nabla f(x) , \nabla p_0(x) \right\rangle \rmd x
= \int_{\mathbb{R}^d} \left\langle \nabla f(x), \nabla U(x) \right\rangle p_0(x) \rmd x \eqsp.
$$
Therefore,
$$
\int_{\mathbb{R}^d} (Lf)(x) p_0(x)\,dx
= \int_{\mathbb{R}^d} \left(- \left\langle \nabla U(x), \nabla f(x) \right\rangle+\Delta f(x)\right)p_0(x) \rmd x
=0 \eqsp ,
$$
for all $f\in C_c^\infty(\mathbb{R}^d)$. Under the
standard generator/core assumptions \citep{ethier1986markov,salez2025markov} for the Langevin diffusion, it implies that for all $t \ge 0$
 $$\pi_{\rm data}P_t=\pi_{\rm data} \eqsp,$$
where $(P_t)_{t \ge0}$ is the semi group associated with \eqref{eq:langevin}.

Then, suppose that $(X_t)_{t\geq 0}$ and $(Y_t)_{t \geq 0}$ both solve \eqref{eq:langevin} driven by the \emph{same} Brownian motion (this is known as synchronous coupling). Setting $Z_t \eqdef X_t-Y_t$, we obtain
$$
\frac{\rmd}{\rmd t}\|Z_t\|^2
= -2\big\langle Z_t, \nabla U(X_t)-\nabla U(Y_t)\big\rangle
\le -2\alpha \| Z_t\|^2.
$$
By Grönwall's lemma, $\|Z_t\|^2\le e^{-2\alpha t}\|Z_0\|^2$, and therefore
\begin{equation} \label{eq:control_intermediate_langevin}
\mathbb{E}\|X_t-Y_t\|^2\le e^{-2\alpha t}\mathbb{E}\|X_0-Y_0\|^2 \eqsp,
\end{equation}
 for any coupling of $(X_0,Y_0)$ of initial conditions. Now choose $Y_0\sim\pi_{\rm data}$. By invariance,
$$
\mathcal L(Y_t)=\pi_{\rm data} \eqsp,
\qquad t\ge0 \eqsp.
$$
Optimizing \eqref{eq:control_intermediate_langevin} over all couplings of
$\mu_0$ and $\pi_{\rm data}$ yields \eqref{eq:wasserstein_contraction_langevin}.
\end{proof}

\paragraph{Unadjusted Langevin Algorithm} 

The diffusion \eqref{eq:langevin} is a continuous-time process, and, except for very specific potentials $U$,
it cannot be simulated exactly. A standard practical surrogate is obtained by discretizing time on the grid
$t_k \eqdef kh$ ($k\in\mathbb N$) and applying an Euler--Maruyama scheme, which yields the
\emph{Unadjusted Langevin Algorithm (ULA)} \citep{parisi1981correlation}:
\begin{equation}\label{eq:ula}
X_{k+1}=X_{k} - h \nabla U(X_{k}) + \sqrt{2h} \xi_{k+1} \eqsp,
\qquad \xi_{k+1} \stackrel{\mathrm{i.i.d.}}{\sim} \mathcal N(0_d,\Id_d) \eqsp.
\end{equation}

This discretization induces a bias known as \emph{discretization error}, but it admits  nonasymptotic error bounds under suitable  assumptions on $U$ \citep{dalalyan2012sparse,durmus2017nonasymptotic, dalalyan2017further, ChengBartlett2018, durmus2019high,VempalaWibisono2019, mou2022improved}.

\begin{figure}[t]
    \centering

    \newlength{\panelheight}
    \setlength{\panelheight}{0.39\textwidth}

    \begin{subfigure}[c]{0.52\textwidth}
        \centering
        \begin{minipage}[c][\panelheight][c]{\textwidth}
            \centering
            \includegraphics[
                width=\textwidth,
                height=\panelheight,
                keepaspectratio
            ]{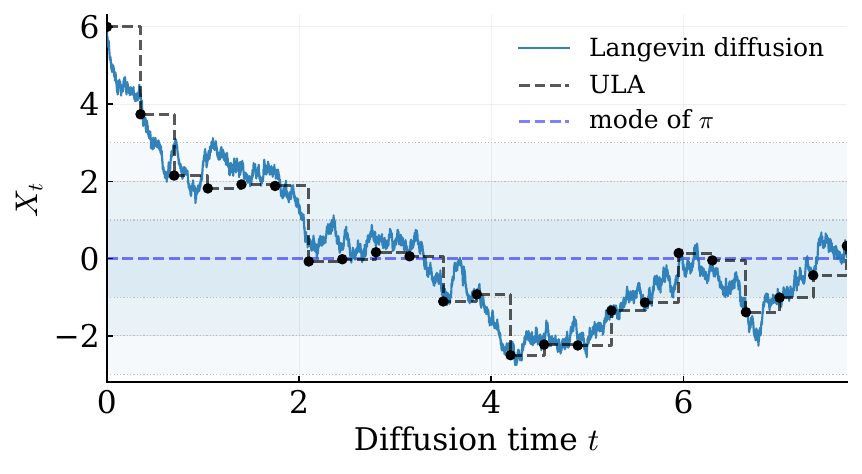}
        \end{minipage}
        \caption{$1$-d example targeting $\mathcal{N}(0,1)$}
        \label{fig:langevin_ula_1d_trajectory}
    \end{subfigure}
    \hfill
    \begin{subfigure}[c]{0.43\textwidth}
        \centering
        \begin{minipage}[c][\panelheight][c]{\textwidth}
            \centering
            \includegraphics[
                width=\textwidth,
                height=\panelheight,
                keepaspectratio
            ]{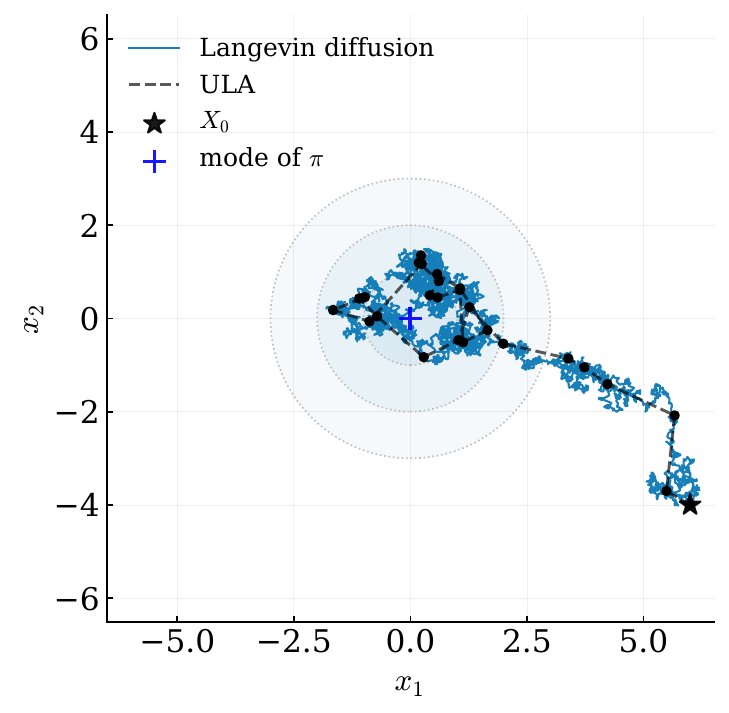}
        \end{minipage}
         \caption{$2$-d example targeting $\mathcal{N}(0,\Id_2)$}
        \label{fig:langevin_ula_2d_trajectory}
    \end{subfigure}

    \caption{
    Langevin diffusion and its Euler--Maruyama discretization in the quadratic case, \emph{i.e.} $U(x)=x^2/2$. The left panel shows the one-dimensional trajectory as a function of time, whereas the right panel shows a two-dimensional trajectory in state space. The blue curve represents a diffusion path, while the black dashed curve and black dots represent the ULA iterates together with their piecewise-constant extension. The two processes are coupled through the same Brownian increments. The pale blue shaded regions indicate typical regions under the stationary distribution.
    }
    \label{fig:langevin_ula_coupled_trajectories}
\end{figure}

\subsection{A learning mechanism: score matching}

\paragraph{Fisher divergence and score matching.} In generative machine learning, sampling is only one side of the story; the other is learning.
As we saw, Langevin sampling requires access to the drift $-\nabla U$, or equivalently to the score function
$\nabla \log p_0$. One may thus fit a parametric vector field
$\score_\theta:\R^d\to\R^d$ parametrized by $\theta \in \Theta$, to approximate this quantity.
A natural discrepancy between score fields is the (squared) $L^2$ error, also known as the
\emph{Fisher divergence}
\begin{equation}\label{eq:fisher_div}
\mathcal L_{\rm explicit}(\theta)
\eqdef
\frac12 \mathbb E \left[\left\|\score_\theta(X_0)-\nabla \log p_0(X_0)\right\|^2\right] \eqsp, \qquad X_0 \sim \pi_{\rm data}\eqsp.
\end{equation}
A key feature of \eqref{eq:fisher_div} is that it does not involve the normalizing constant:
if $p_0(x)=Z^{-1}\exp(-U(x))$, then the constant $Z$ disappears upon differentiation. Once an approximation of the score is learned, one may sample from $\pi_{\rm data}$ by discretizing the Langevin diffusion \eqref{eq:ula} after replacing $\nabla \log p_0$ by $\score_\theta$. 

A key observation due to \citet{hyvarinen2005estimation} is that, although the Fisher divergence
\eqref{eq:fisher_div} involves the unknown score $\nabla\log p_0$, it can be rewritten (up to an additive
constant independent of $\theta$) as an expectation under $\pi_{\rm data}$ of terms depending only on $\score_\theta$ and
its divergence. Since $\pi_{\rm data}$ is, precisely, available through i.i.d. samples, this expectation can be approximated
by Monte Carlo methods. The next proposition makes this reduction explicit via an integration-by-parts
argument.

\begin{proposition}[\citealp{hyvarinen2005estimation}]\label{prop:hyvarinen}
Assume that $p_0\in C^1(\R^d)$ and $\score_\theta\in C^1(\R^d;\R^d)$.
Assume moreover that integration by parts is valid and boundary terms vanish
(e.g.\ $p_0(x)\score_\theta(x)\to 0$ as $\|x\|\to\infty$), and that the expectations below are finite. Then, for every $\theta\in\Theta$,
\begin{equation}\label{eq:hyvarinen_identity}
\mathcal L (\theta)
=
\mathbb E \left[\tfrac12\|\score_\theta(X_0)\|^2 + \mathrm{div} \left( \score_\theta \right) (X_0) \right]
+ C \eqsp, \qquad X_0 \sim \pi_{\rm data} \eqsp,
\end{equation}
where $C$ does not depend on $\theta$. 
\end{proposition}

\begin{proof}
Expand the square in \eqref{eq:fisher_div}:
$$
\mathcal L(\theta) =\tfrac12 \int p_0(x)\|\score_\theta(x)\|^2\rmd x
- \int p_0(x) \left\langle \score_\theta(x), \nabla\log p_0(x) \right\rangle \rmd x
+ C \eqsp.
$$
For the cross term, using that $\nabla\log p_0 = (\nabla p_0)/p_0$,
$$
\int p_0(x) \score_\theta(x)^\top \nabla\log p_0(x) \rmd x
=
\int \score_\theta(x)^\top \nabla p_0(x) \rmd x \eqsp.
$$
By integration by parts and the boundary condition,
$$
\int \left\langle \score_\theta(x), \nabla p_0 (x) \right\rangle \rmd x
=
-\int_{\R^d} \left( \operatorname{div}  \score_\theta \right) (x) \eqsp p_0(x) \rmd x \eqsp,
$$
which yields \eqref{eq:hyvarinen_identity}.
\end{proof}

\begin{remark}
The proof above highlights why the quadratic loss ($p=2$) is special: expanding the square produces a cross-term
$\E_{p_0}[\langle \score_\theta(X) ,\nabla\log p_0(X) \rangle]$ that is linear in $\nabla\log p_0$ and can therefore be eliminated by integration by parts.
For $p\neq 2$, the dependence on $\nabla\log p_0$ becomes nonlinear and no comparable simplification is available in general.
\end{remark}

In light of Proposition~\ref{prop:hyvarinen}, minimizing $\mathcal L_{\rm explicit}(\theta)$ over $\theta\in\Theta$
is equivalent to minimizing the $\theta$-dependent part of \eqref{eq:hyvarinen_identity}. This motivates the
\emph{score matching} objective
$$
\mathcal{L}_{\rm SM} (\theta) = \mathbb E \left[\tfrac12 \|\score_\theta(X_0)\|^2+ \left( \operatorname{div} \score_\theta \right) (X_0)\right] \eqsp, \qquad X_0 \sim \pi_{\rm data} \eqsp.
$$
In practice, given $n$ i.i.d samples $\{X_i \}_{i=1}^n \stackrel{\mathrm{i.i.d.}}{\sim} \pi_{\rm data}$, one replaces the expectation by its Monte Carlo approximation
$$
\widehat{\mathcal L}_{{\rm SM}}^n (\theta)
\eqdef
\frac1n \sum_{i=1}^n \left(\frac12\|\score_\theta(X_i)\|^2+ \left(  \operatorname{div} \score_\theta \right) (X_i) \right) \eqsp.
$$
This finite-sample objective can then be optimized with gradient-based methods;
in large-scale settings, one typically uses stochastic-gradient methods based on
minibatches \citep{bottou2018optimization}. However, this direct form of score matching is not widely used in large-scale modeling because it is computationally intensive. Computing
$
\operatorname{div}\score_\theta(x)
=
\operatorname{Tr}(\nabla_x\score_\theta(x))
$
requires access to the trace of the input Jacobian of the score network, which is extremely expensive in high dimension \citep{martens2012estimating,song2019generative,
song2020sliced}. There are also statistical reasons to avoid such approach. For empirical measures, or for distributions
concentrated near low-dimensional sets, this score may be ill-defined or highly singular. These two difficulties motivate looking for the regularized approaches introduced next.

%\textcolor{red}{Maybe discuss sliced score matching just for comprehensiveness ? }

% A common approximation uses the Hutchinson identity:
% for $v$ such that $\mathbb E[vv^\top]=\Id$,
% $$
% \nabla\cdot \score_\theta(x)
% =\mathbb E_v\big[v^\top (\nabla_x \score_\theta(x))v\big]
% \approx v^\top (\nabla_x \score_\theta(x))v,
% $$
% which reduces the cost to a constant number of automatic differentiation passes.

\paragraph{The conditioning trick.} A computationally cheaper alternative to direct score matching is to introduce an auxiliary \emph{corruption} mechanism, specified by a
Markov kernel that is known analytically and that can be sampled from. Concretely, let $\tilde q(\cdot\mid x_0)$ be a
Markov kernel on $\R^d$. Given $X_0\sim \pi$, define the corrupted variable $\tilde X$ by
$$
\tilde X \mid X_0=x_0 \sim \tilde q(\cdot\mid x_0) \eqsp,
\qquad
\tilde p(x)\eqdef \int_{\R^d} \tilde q(x\mid x_0)\,\pi(\rmd x_0) \eqsp,
$$
where $\tilde p$ denotes the density of $\tilde X$ (when it exists). Although $\pi$ is only accessible through
samples, the pair $(X_0,\tilde X)$ can be generated by first drawing $X_0$ and then sampling
$\tilde X\sim \tilde q(\cdot\mid X_0)$. This construction yields a tractable route to learn the marginal score
$\nabla\log \tilde p$ of the corrupted variable $\tilde X$ via the conditional score function $\nabla \log \tilde q(\tilde X\mid X_0)$.

\begin{lemma} \label{lem:marginal_score_condexp}
Assume that for $\pi$-a.e.\ $x_0$, the kernel $\tilde q(\cdot\mid x_0)$ admits a $C^1$ density and that
differentiation under the integral sign is justified. Assume moreover that
$\tilde p(x)>0$ on the region considered. Then, for $\tilde X\sim \tilde p$,
\begin{equation}\label{eq:marginal_score_condexp}
\nabla \log \tilde p(\tilde X)
=
\mathbb E \left[\nabla \log \tilde q(\tilde X\mid X_0) \middle| \tilde X\right] \eqsp.
\end{equation}
\end{lemma}

\begin{proof}
Differentiating under the integral sign gives
$$
\nabla \tilde p(x)=\int_{\R^d} \nabla_x \tilde q(x\mid x_0) \pi (\rmd x_0)
=\int_{\R^d}\tilde q(x\mid x_0)\,\nabla_x\log \tilde q(x\mid x_0) \pi (\rmd x_0) \eqsp.
$$
Dividing by $\tilde p(x)$ yields \eqref{eq:marginal_score_condexp}.
\end{proof}

Lemma~\ref{lem:marginal_score_condexp} suggests learning the intractable marginal score $\nabla\log \tilde p$  by regression against the tractable
conditional score $\nabla \log \tilde q(\tilde X\mid X_0)$. Indeed, for any square-integrable random vector $Y$,
$$
\mathbb E\big[ Y | \tilde X \big] =\arg\min_{f}\ \mathbb E\|f(\tilde X)-Y\|^2 \eqsp,
$$
where the minimum is over measurable $f$ such that $\mathbb E \|f(\tilde X) \|^2<\infty$. Taking $Y=\nabla \log \tilde q(\tilde X\mid X_0)$ and using Lemma~\ref{lem:marginal_score_condexp}, the minimizer is
$f^\star (\tilde X)=\nabla\log \tilde p(\tilde X)$. 

This, in turns, justifies the choice of the $L_2$ distance as natural discrepancy to optimize in order to learn the score function $\nabla \log \tilde p$ and is further formalized by the following proposition.

\begin{proposition}\label{prop:csm_single}
Let $X_0\sim \pi$ and $\tilde X\mid X_0=x_0\sim \tilde q(\cdot\mid x_0)$, and denote by $\tilde p$ the density of $\tilde X$. Assume that the identity in Lemma~\ref{lem:marginal_score_condexp} holds.
For a parametric model $\score_\theta:\R^d\to\R^d$, define
$$
\mathcal L_{\rm CSM}(\theta)
\eqdef
\tfrac12\,\E\Big[\big\|\score_\theta(\tilde X)-\nabla_x\log \tilde q(\tilde X\mid X_0)\big\|^2\Big] \eqsp.
$$
Then,
$$
\mathcal L_{\rm CSM}(\theta)
=
\tfrac12 \E \big\|\score_\theta(\tilde X)-\nabla\log \tilde p(\tilde X)\big\|^2
+ C \eqsp,
$$
where $C$ does not depend on $\theta$.
\end{proposition}

\begin{proof}[Sketch of proof]
Suppose for simplicity that $\mathbb E||\score_{\theta}( \tilde X)||^2 < \infty$ and $\mathbb E||\nabla \log \tilde q(\tilde X\mid X_0)||^2 < \infty$. Let $ Y \eqdef \nabla \log \tilde q(\tilde X\mid X_0)$ and $m(\tilde X)\eqdef \mathbb E[Y\mid \tilde X]=\nabla\log \tilde p(\tilde X)$
by Lemma~\ref{lem:marginal_score_condexp}. The Pythagorean identity for conditional expectation gives
$$
\mathbb E\|\score_\theta(\tilde X)- Y\|^2
=
\mathbb E\|\score_\theta(\tilde X)-m(\tilde X)\|^2 + \mathbb E\|m(\tilde X)-Y\|^2 \eqsp,
$$
and the last term is independent of $\score_\theta$.
\end{proof}

\begin{remark}
In the special case where $\tilde q(\cdot\mid x_0)$ is a Gaussian corruption
kernel, Proposition \ref{prop:csm_single} is the denoising score matching identity of
\citet{Vincent}.
\end{remark}

In practice, if $\{X_i\}_{i=1}^n$ are i.i.d. data samples and
$\tilde X_i\sim \tilde q(\cdot\mid X_i)$ are independent corruptions, the population loss is replaced by the empirical objective
$$
\widehat{\mathcal L}_{{\rm CSM}}^n(\theta)
\eqdef \frac1{2n}\sum_{i=1}^n \left\| \score_\theta(\tilde X_i) - \nabla_x\log \tilde q(\tilde X_i\mid X_i) \right\|^2 \eqsp,
$$
and this objective can then be optimized by stochastic-gradient methods.

\subsection{Score matching and Gaussian corruption kernels} \label{sec:denoising_score_matching}

\paragraph{Denoising score matching.} The conditional formulation of the previous section provides a tractable way to learn score functions, but it
targets the marginal of a \emph{corrupted} variable rather than the original data distribution $\pi_{\rm data}$. A standard way to construct such corruptions is to consider additive Gaussian kernels \citep{Vincent}. For $\sigma>0$, let
\begin{equation} \label{eq:gaussian_smoothed_var}
\tilde X_\sigma \eqdef X_0 + \sigma Z.
\qquad
Z\sim\mathcal N(0,\Id_d),\quad X_0\sim \pi_{\rm data},\quad Z\perp\!\!\!\perp X_0 \eqsp,
\end{equation}
and denote by $\pi_{\sigma}$ the distribution of $\tilde X_\sigma$. Several features make Gaussian corruption a particularly natural choice.

\begin{itemize}[leftmargin=*]

\item \textbf{Closed-form conditional score.} The Gaussian corruption kernel is analytically tractable and simple. The conditional density of $\tilde X_\sigma$ given $X_0=x_0$ is
$$
q_\sigma(x\mid x_0)
=
(2\pi\sigma^2)^{-d/2}
\exp \left(-\frac{\|x-x_0\|^2}{2\sigma^2}\right) \eqsp.
$$
and its score function is given by
\begin{equation} \label{eq:score_function_closed_form}
\nabla_x\log q_\sigma(x\mid x_0)
=
-\frac{x-x_0}{\sigma^2} \eqsp.
\end{equation}
Thus, the conditional score appearing in the conditional score matching identity in Proposition \ref{prop:csm_single} is available in closed form. Moreover, its score is simple, linear in $x$. This is a distinctive property of Gaussian kernels: affine scores characterize Gaussian densities on $\R^d$ \citep[Proposition 5]{pierret2025diffusion}.

\item \textbf{Computational convenience.} Gaussian corruption kernels are computationally efficient to sample from. Gaussian random variables generation is heavily optimized in various numerical libraries, through methods like variant of Box--Muller transforms \citep{box1958generation} in \texttt{PyTorch} or Ziggurat methods \citep{marsaglia2000ziggurat} in \texttt{NumPy}. This makes Gaussian corruption particularly convenient in large-scale learning procedures.

\item \textbf{Regularization properties.} Gaussian corruption regularizes the data distribution. The score $\nabla\log p_0$ is only defined under absolute continuity and smoothness assumptions on $p_0$. In particular, for the empirical measure $\hat\pi_N=\frac1N\sum_{i=1}^N \delta_{X_i}$ there is no density with respect to Lebesgue measure, hence no classical score. More generally, if $\pi_{\rm data}$ is supported on (or highly concentrated near) a low-dimensional subset of $\R^d$, the score may be ill-defined or extremely difficult to estimate. By contrast, for every $\sigma>0$, the
law of $\tilde X_\sigma$ admits a positive density given by
$$
p_\sigma(x)=\int \varphi_\sigma(x-x_0) \pi_{\rm data}(\rmd x_0) \eqsp,
\qquad
\varphi_\sigma = \mathcal N(0,\sigma^2 \Id_d) \eqsp,
$$
Moreover, $p_\sigma\in C^\infty(\R^d)$ for every $\sigma>0$, so that the
(smoothed) score $\nabla\log p_\sigma$ is well-defined and typically much better behaved than $\nabla\log p_0$. Finally, Gaussian corruption also preserves finite moments: for any $p\ge 1$, if $\E\|X_0\|^p<\infty$, then $\E\| \tilde X_\sigma \|^p<\infty$.

\item \textbf{Connection with denoising.} Finally, Gaussian corruption turns score estimation into a denoising problem.
Indeed, Tweedie's formula \citep{robbins1956empirical,miyasawa1961empirical}
gives
\begin{equation}\label{eq:tweedie}
\mathbb E[X_0\mid \tilde X_\sigma=y] = y + \sigma^2 \nabla \log p_\sigma(y) \eqsp,
\end{equation}
or equivalently,
$$
\nabla\log p_\sigma(y) = \frac{\mathbb E [X_0 \mid \tilde X_\sigma=y] - y}{\sigma^2} \eqsp.
$$
In other words, $\nabla \log p_\sigma(y)$ determines the Bayes-optimal denoiser under squared $L_2$ loss. This link can also be seen directly at the level of the conditional score matching objective. Using, \eqref{eq:score_function_closed_form}, the conditional score matching loss at noise level $\sigma$ writes as
$$
\mathcal L_{\rm DSM}(\theta;\sigma)
= \frac12 \mathbb E\left[ \left\| \score_\theta(\tilde X_\sigma,\sigma) + \frac{\tilde X_\sigma-X_0}{\sigma^2}
\right\|^2 \right] \eqsp.
$$
Defining the denoiser $D_\theta(y,\sigma) \eqdef y+\sigma^2\score_\theta(y,\sigma)$, the above writes as
$$
\mathcal L_{\rm DSM}(\theta;\sigma) = \frac{1}{2\sigma^4} \mathbb E\left[
\|D_\theta(\tilde X_\sigma,\sigma)-X_0\|^2 \right] \eqsp.
$$
Hence, up to a multiplicative factor, denoising score matching is equivalent to training a denoiser by mean-square error. At the population level, the optimal denoiser is $D_\sigma^\star(y) = \mathbb E[X_0\mid \tilde X_\sigma=y]$. This interpretation is important both theoretically and in practice. In practice, score learning can be implemented
as a denoising task, for which modern neural architectures, in particular
U-Net-type architectures for image data, are especially well suited \citep{Vincent,ronneberger2015u,ho2020denoising,beatsGAN, edmkarras}. In theory, it turns score estimation into a regression problem with noisy labels, which can be studied through optimization and generalization theory for neural networks \citep{han2024neural,strasman2026sgd}.
\end{itemize}

\paragraph{Gaussian smoothing.}

Denoising score matching at a fixed $\sigma>0$ provides an estimate of the score $\nabla\log p_\sigma$ of the Gaussian-smoothed variable $\tilde X_\sigma$ defined in \eqref{eq:gaussian_smoothed_var}. Consequently, any Langevin-type sampler driven by this score
targets $\pi_\sigma$ rather than the original data distribution $\pi_{\rm data}$. Thus a single noise level introduces a bias.

Choosing $\sigma$ is therefore a central question. On the one hand, small $\sigma$ keeps the smoothed distribution close to the data distribution. For example, in $\mathcal{W}_2$ distance,
$$
\mathcal{W}_2(\pi_{\rm data}, \pi_{\sigma})^2 \leq \mathbb{E} \left[ \left\| X_0 - \tilde X_\sigma \right\|^2 \right]= \sigma^2 \mathbb{E} \left[ \| Z \|^2 \right] = \sigma^2 d
$$
On the other hand, if $\sigma$ is too small, then the denoising score matching target becomes poorly conditioned. Indeed, for the Gaussian corruption kernel
$$
\mathbb E \left[
\left\|\nabla_x \log q_\sigma(\tilde X_\sigma\mid X_0)\right\|^2
\right]
= \frac{1}{\sigma^2}\mathbb E\|Z\|^2 = \frac{d}{\sigma^2} \eqsp.
$$
Thus, as $\sigma\downarrow0$, the target amplitude explodes, which can make
score estimation statistically and numerically difficult.

There is also a second, distinct drawback to taking $\sigma$ too small. It may impair the mixing properties of the Langevin sampler used at generation. Although Gaussian convolution makes $p_\sigma$ strictly positive on all of $\R^d$, for small $\sigma$ the distribution $p_\sigma$ remains close to $p_0$ and may therefore retain a strongly multimodal geometry. In particular, if $p_0$ is supported on well-separated clusters, then $p_\sigma$ only creates very narrow and low-density bridges between these regions when $\sigma$ is small (see the toy example below). As a consequence, the potential
$$
U_\sigma(x)\eqdef -\log p_\sigma(x)
$$
may still exhibit deep wells separated by high barriers. Langevin dynamics targeting $p_\sigma$ can then be \emph{metastable} \citep{bovier2004metastability}: it may converge rapidly to local equilibrium within one mode, while transitions between distinct modes occur on a much longer time scale. Increasing $\sigma$ smooths these barriers and generally
improves the global geometry seen by the sampler, but at the price of introducing more bias.

\paragraph{A simple two-cluster toy model example.} Let $\varphi_\sigma(x)$ denotes the density of $\mathcal N(0,\sigma^2)$. In one dimension, consider the data distribution
\begin{equation} \label{eq:pi_data_two_clusters}
\pi_{\rm data}= \tfrac12 \delta_{-a} + \tfrac12 \delta_{+a} \eqsp ,\qquad a>0 \eqsp.
\end{equation}
After Gaussian smoothing by $\mathcal N(0,\sigma^2)$, the density of the smoothed random variable is
$$
p_\sigma(x)
=
\tfrac12 \varphi_\sigma(x-a) + \tfrac12 \varphi_\sigma(x+a) \eqsp.
$$
For $\sigma\ll a$, this density remains strongly bimodal. Its score can be written explicitly as
$$
\partial_x \log p_\sigma (x)
=
-\frac{x}{\sigma^2}
+
\frac{a}{\sigma^2}\tanh\!\left(\frac{ax}{\sigma^2}\right) \eqsp.
$$
In particular,
$$
\partial_x \log p_\sigma (0)=0 \eqsp,
\qquad
\partial_x^2 \log p_\sigma(0)=\frac{a^2-\sigma^2}{\sigma^4}
\sim \frac{a^2}{\sigma^4}
\qquad(\sigma\downarrow 0) \eqsp.
$$
Thus, near the low-density region separating the two modes, the score becomes increasingly steep as
$\sigma$ decreases. Equivalently, if $U_\sigma\eqdef -\log p_\sigma$, then the barrier height between a mode and the saddle at $0$ satisfies
$$
U_\sigma(0)-U_\sigma(a)
=
\log \left(\frac{p_\sigma(a)}{p_\sigma(0)}\right)
=
\frac{a^2}{2\sigma^2}-\log 2+o(1) \eqsp,
\qquad \sigma\downarrow 0 \eqsp.
$$
This illustrates why a single small noise level can be problematic: although $p_\sigma$ is smooth, the corresponding
Langevin dynamics may still face sharp barriers and poor mixing between modes. This behavior is illustrated in
Figure~\ref{fig:two_cluster_gaussian_smoothing}.

\begin{figure}[t]
    \centering

    \begin{subfigure}{0.48\textwidth}
        \centering
        \includegraphics[width=\textwidth]{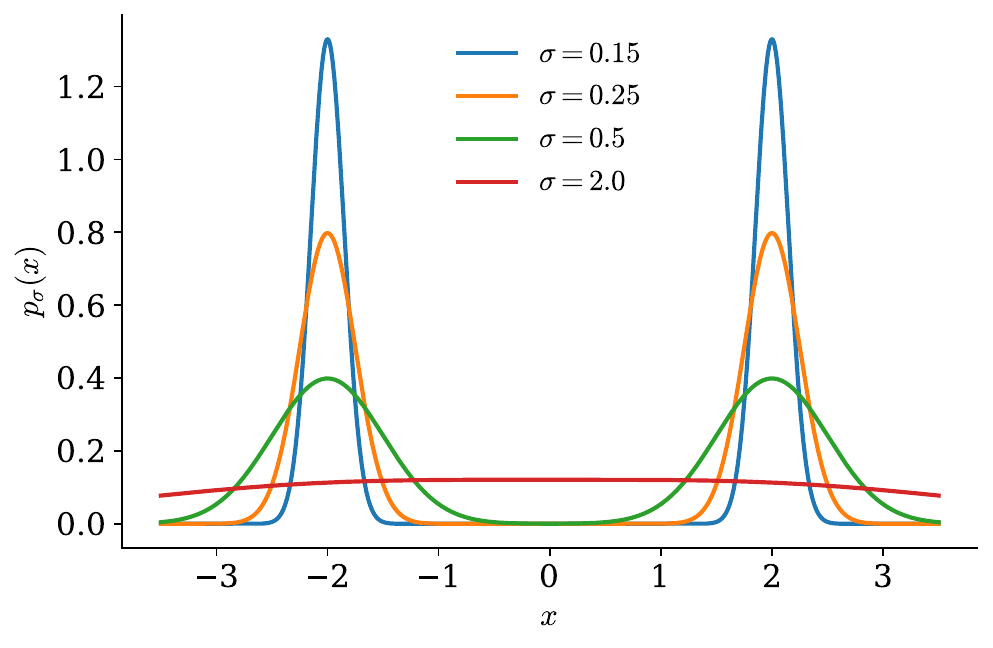}
        \caption{Density $p_\sigma$.}
    \end{subfigure}
    \hfill
    \begin{subfigure}{0.48\textwidth}
        \centering
        \includegraphics[width=\textwidth]{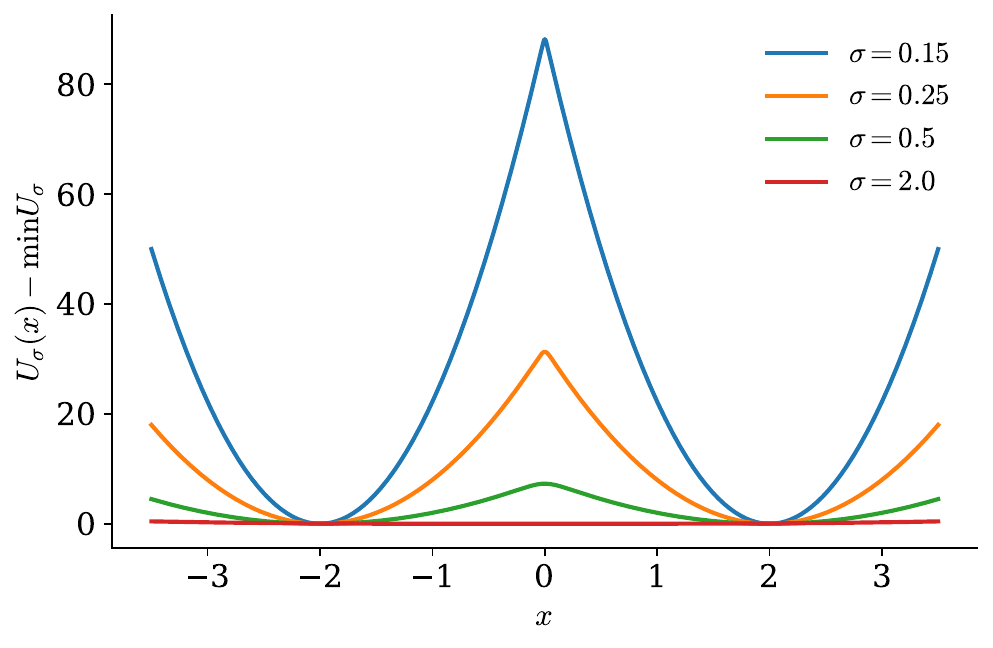}
        \caption{Potential $U_\sigma=-\log p_\sigma$.}
    \end{subfigure}

    \vspace{1em}

    \begin{subfigure}{0.5\textwidth}
        \centering
        \includegraphics[width=\textwidth]{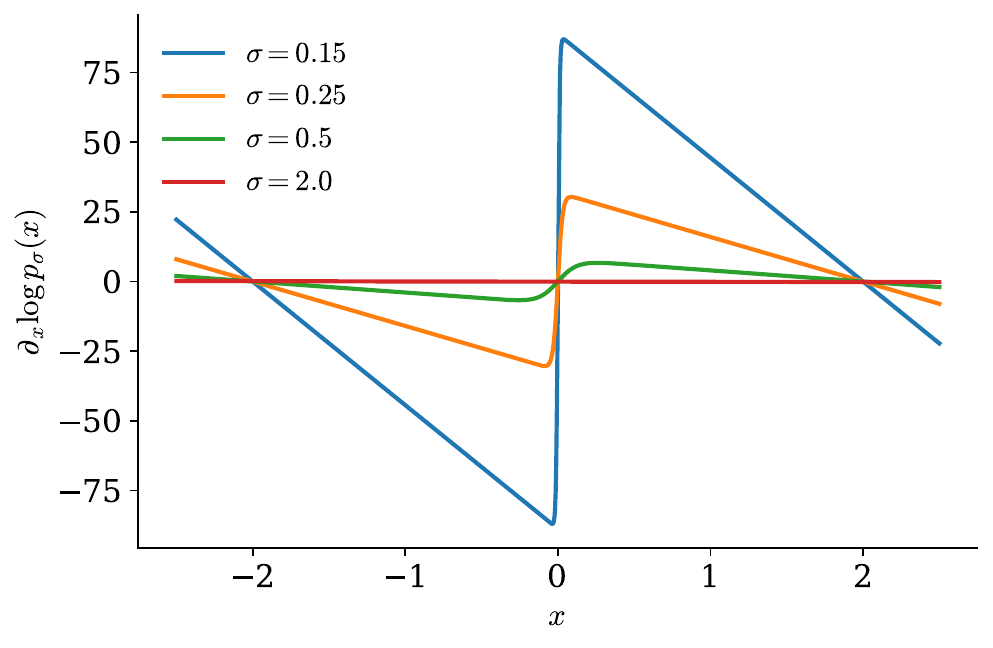}
        \caption{Score function $\partial_x\log p_\sigma$.}
    \end{subfigure}

    \caption{
    Gaussian smoothing of the two-cluster distribution
    $\pi_{\rm data}=\frac12\delta_{-a}+\frac12\delta_{+a}$ with $a=2$.
    Small noise preserves multimodality and produces sharp barriers.
    In the middle panel we plot $U_\sigma-\min U_\sigma$ rather than $U_\sigma$ itself; this vertical shift is only for visualization.
    }
    \label{fig:two_cluster_gaussian_smoothing}
\end{figure}

\subsection{Annealed Langevin dynamics}

These observations indicate that a single noise level is intrinsically suboptimal: large $\sigma$ yields substantial bias, while small $\sigma$ leads to poorly scaled learning targets and can produce sharp energy barriers in the sampling process. This trade-off is one of the main reasons why modern score-based methods do not rely on a single noise level. 

At the algorithmic level, learning scores at multiple noise levels leads to
\emph{annealed Langevin dynamics} \citep{neal2001annealed, song2019generative,song2020improved,jolicoeur-martineau2021adversarial}. Given a decreasing sequence of noise levels
$$
\sigma_L>\sigma_{L-1}>\cdots>\sigma_1 \eqsp,
$$
assume that score models have been trained so that
$$
\score_\theta(\cdot,\sigma_i)\approx \nabla\log p_{\sigma_i},
\qquad i=1,\ldots,L \eqsp.
$$
Annealed Langevin dynamics proceeds from large to small noise levels. At each level $\sigma_i$, the noise is frozen and one performs $K_i\in\mathbb N$ Langevin steps (similar to \eqref{eq:ula}) using the score field $\score_\theta(\cdot,\sigma_i)$. For fixed $\sigma_i$, these steps locally move the sample toward the smoothed density $p_{\sigma_i}$. One then decreases the noise level and uses the output of the previous stage as initialization for the next one. The procedure is precisely described in Algorithm \ref{alg:ald}. It follows that the sampler starts from a highly smoothed distribution and progressively refines the samples toward lower noise distributions closer to the data distribution (see Figure \ref{fig:ald-vs-reverse-sde}).

% \begin{algorithm}[t]
% \caption{Annealed Langevin dynamics}
% \label{alg:ald}
% \begin{algorithmic}[1]
% \Require Noise levels $\sigma_L>\sigma_{L-1}>\cdots>\sigma_1$
% \Require Learned score approximation $\score_\theta(\cdot,\sigma_i)\approx \nabla\log p_{\sigma_i}$
% \Require Step sizes $\eta_i>0$ and numbers of steps $K_i\in\mathbb N$
% \State Initialize $x\sim \pi_{\rm init}$, where $\pi_{\rm init}$ is a high-noise distribution
% \For{$i=L,L-1,\ldots,1$}
%     \For{$k=0,\ldots,K_i-1$}
%         \State Sample $\xi_{i,k+1}\sim\mathcal N(0,\Id_d)$
%         \State
%         \[
%         x \leftarrow x
%         +\eta_i\,\score_\theta(x,\sigma_i)
%         +\sqrt{2\eta_i}\,\xi_{i,k+1}
%         \]
%     \EndFor
% \EndFor
% \State \Return $x$
% \end{algorithmic}
% \end{algorithm}

\begin{algorithm}[t]
\caption{Annealed Langevin dynamics}
\label{alg:ald}
\begin{algorithmic}[1]
\Statex \textbf{Input:} noise levels $\sigma_L>\sigma_{L-1}>\cdots>\sigma_1$;
learned score models $\score_\theta(\cdot,\sigma_i)\approx \nabla\log p_{\sigma_i}$;
step sizes $\eta_i>0$; numbers of steps $K_i\in\mathbb N$
\Statex \textbf{Initialization:} sample $x\sim \pi_{\rm init}$.
\For{$i=L,L-1,\ldots,1$}
    \For{$k=0,\ldots,K_i-1$}
        \State Sample $Z_{i,k+1} \stackrel{\mathrm{i.i.d.}}{\sim} \mathcal N(0,\Id_d)$
        \State $x \gets x+\eta_i\,\score_\theta(x,\sigma_i)+\sqrt{2\eta_i} Z_{i,k+1}$
    \EndFor
\EndFor
\State \Return $x$
\end{algorithmic}
\end{algorithm}

\paragraph{Limitations of ALD.}

Annealed Langevin dynamics provides a natural discrete implementation of a multi-scale noising idea, but it also has several limitations. First, the sampler is intrinsically sequential. At each noise level $\sigma_i$, the output of the Langevin chain becomes the initial condition for the next chain at
$\sigma_{i-1}$. Thus any error made at an early stage, whether due to time discretization, an imperfect score estimate, or insufficient mixing, is inevitably carried forward to lower noise levels. Such errors may accumulate over time and degrade the final sample quality.

Second, the method can be computationally expensive. If $K_i$ Langevin steps are
used at level $\sigma_i$, then one sample requires $\sum_{i=1}^L K_i$
sequential evaluations of the score network. 

Third, each Langevin step is local. For a fixed noise level $\sigma_i$, the
Langevin chain targets $p_{\sigma_i}$ only in the ideal limit of exact scores,
vanishing discretization error, and sufficiently long simulation time. In high
dimension, or when $p_{\sigma_i}$ remains multimodal, local moves may mix slowly between well-separated regions.

\paragraph{Density interpolation versus transport.}
There is also a more conceptual limitation, which remains even if one ignores
score approximation errors and time-discretization errors. Annealed Langevin dynamics starts from a prescribed sequence of intermediate
densities, informally when $\sigma_1 >0 $ is small
$$
p_{\sigma_L}, \eqsp p_{\sigma_{L-1}}, \eqsp \ldots, \eqsp p_{\sigma_1} \approx \pi_{\rm data} \eqsp ,
$$
and uses Langevin dynamics to partially equilibrate the sample at each noise
level. Thus it is naturally viewed as a discrete interpolation in the space of densities, together with local MCMC updates. However, specifying a path of
densities is not the same as specifying a stochastic process whose marginals
follow that path. To see this, imagine refining the noise ladder into a smooth schedule
$t\mapsto \sigma(t)$ and write $ p_t \eqdef p_{\sigma(t)}$ (resp. $\pi_t$ for the distribution law).
A naive continuous analogue of annealed Langevin dynamics would be the
time-inhomogeneous Langevin process
\begin{equation} \label{eq:langevin_multi_step}
\rmd Y_t = \nabla\log p_t(Y_t) \rmd t + \sqrt{2} \rmd B_t \eqsp.
\end{equation}
For smooth compactly supported test functions $f$, It\^o's formula gives the
time-dependent infinitesimal generator
$$
L_t f(x)
\eqdef
\lim_{h\downarrow 0}
\frac{\mathbb E\!\left[f(Y_{t+h}^{t,x})\right]-f(x)}{h} =  \langle \nabla\log p_t,\nabla f\rangle+\Delta f \eqsp, 
$$
where $Y^{t,x}$ denotes the solution started from $x$ at time $t$. Let $L_t^*$ denote its formal adjoint with respect to the Lebesgue measure.\footnote{That is,
$$
\int (L_t f)(x)\rho(x) \rmd x = \int f(x)(L_t^*\rho)(x) \rmd x
$$
for $f\in C_c^\infty(\mathbb R^d)$, where the expression for $L_t^*\rho$ follows
by integration by parts.} Then, for
a smooth density $\rho$,
$$
L_t^*\rho = - \operatorname{div} \left( \rho \nabla\log p_t \right) +\Delta \rho \eqsp.
$$
In particular, for each fixed time $t$,
$$
L_t^*p_t
= - \operatorname{div} \left( p_t\nabla\log p_t \right)+\Delta p_t
= - \operatorname{div} \left( \nabla p_t \right) +\Delta p_t
= 0 \eqsp.
$$
Thus $\pi_t$ is invariant for the Langevin dynamics obtained by \emph{freezing the score at time $t$}. However, if $\rho_t$ denotes the time marginal distribution of \eqref{eq:langevin_multi_step}, then $\rho_t$ must satisfy the Fokker--Planck equation
$$
\partial_t\rho_t=L_t^*\rho_t \eqsp.
$$
If one imposed $\rho_t=p_t$, the previous identity would give
$$
\partial_t p_t=L_t^*p_t=0 \eqsp,
$$
which is incompatible with a nonconstant density path. Hence, the process driven only by the instantaneous score $\nabla\log p_t$ as \eqref{eq:langevin_multi_step} does not realize $p_t$ as its marginal law. The point is that \eqref{eq:langevin_multi_step} describes local equilibration around a
frozen density; it does not specify how mass should move from $p_t$ to
$p_{t+\rmd t}$. In other words, a family of densities $(p_t)_{t\in[0,T]}$ specifies the desired distribution at each time, but it does not determine a stochastic process on path space.

From this viewpoint, annealed Langevin dynamics is best understood as a sequence of partial equilibration steps at successive noise levels rather than a transport of measure on path space (see Figure \ref{fig:ald-vs-reverse-sde}). This motivates passing from a marginal-level interpolation to a process-level formulation. 

\begin{figure}[t]
    \centering
    \includegraphics[width=\textwidth]{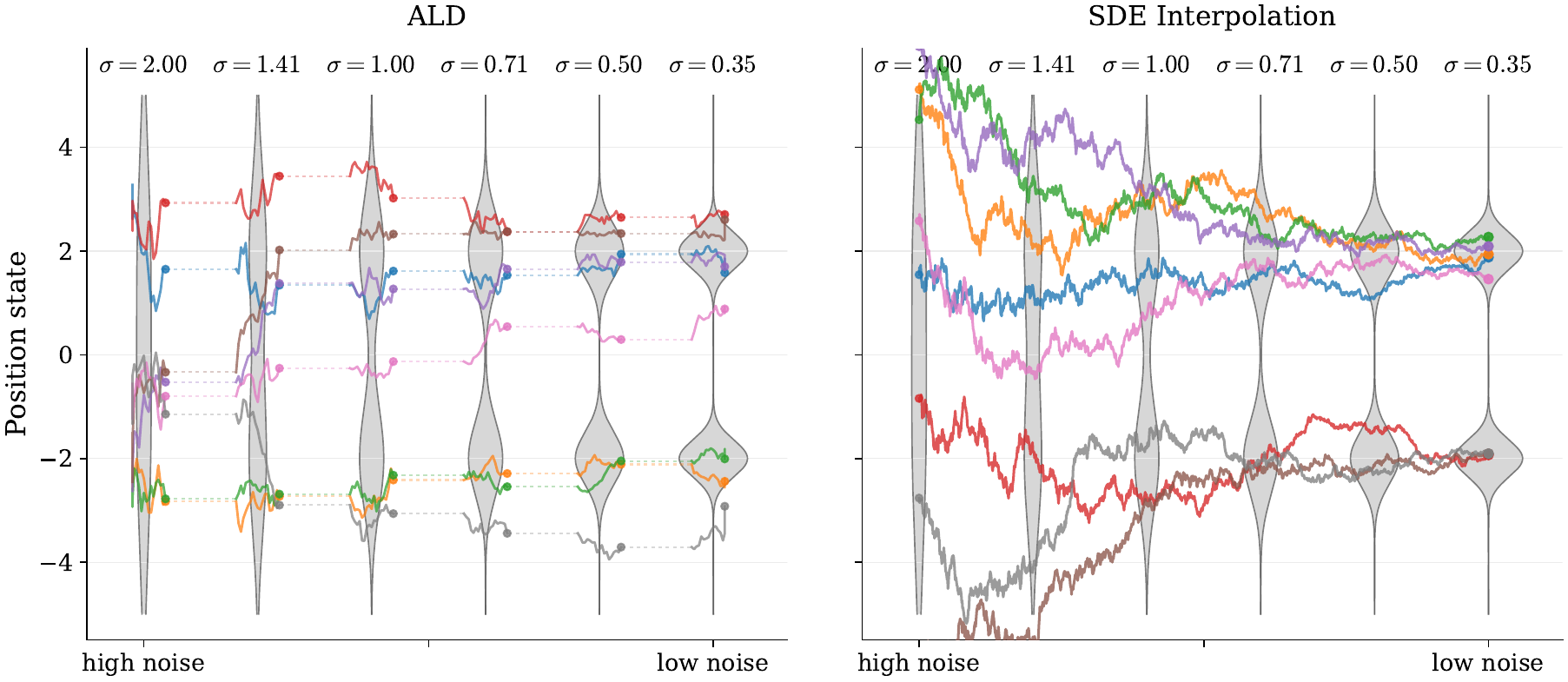}
    \caption{
Annealed Langevin Dynamics versus diffusion interpolation on the one-dimensional two-cluster example ($a=2$). Both panels show the
same smoothed densities $p_{\sigma(t)}$ as grey slices. Left: annealed
Langevin dynamics applies local Langevin moves at fixed noise levels
$\sigma_i$, treating each $p_{\sigma_i}$ as a frozen target. Right: the
SDE formulation gives a continuous stochastic dynamics from high noise
to low noise. 
}
\label{fig:ald-vs-reverse-sde}
\end{figure}
%%%% a checker maybe
% \url{https://arxiv.org/pdf/2511.10406}

%\textcolor{\red}{Langevin and score models}

\section{The time-reversal perspective}

We now pass from marginal smoothing to path-space dynamics. The previous section showed that a family of smoothed densities can be used as training targets for score estimation, but that it does not by itself specify how particles should move across time. To obtain a generative dynamics, we introduce a forward noising process, initialized at the data distribution, and then study its time reversal. This viewpoint turns generative modeling into a problem of transporting probability measures along a stochastic dynamics: the forward process transports data toward a tractable high-noise distribution, while the reverse process transports samples back toward the data distribution.

\subsection{Forward noising processes}

In this section, we first describe a class of Gaussian noising mechanisms whose one-time marginals interpolate between the data distribution and a
high-noise distribution. We then explain how such marginal paths can be realized as the marginals of Markov diffusion processes $(\ora X_t)_{t\in[0,T]}$.

\paragraph{Affine Gaussian smoothing.}

We first describe Gaussian corruptions at the level of conditional laws. This is slightly more general than the scalar noising mechanisms used in most
SGMs, and it makes clear why Gaussian corruptions are
convenient.
Let $X_0\sim \pi_{\rm data}$ be a data variable in $\mathbb R^d$. An affine
Gaussian corruption is specified by deterministic maps
$$
A_t \in \mathbb R^{d\times d},
\qquad
b_t \in \mathbb R^d,
\qquad
\Sigma_t = \Gamma_t\Gamma_t^\top \in \mathbb R^{d \times d},
$$
with $\Sigma_t$ symmetric nonnegative definite. We define a corrupted variable
$X_t\in \mathbb R^d$ as
$$
 X_t = A_t X_0 + b_t + \Gamma_t Z \eqsp,
\qquad Z \sim \mathcal N(0,\Id_d) \eqsp, \qquad Z \perp\!\!\!\perp X_0 \eqsp,
$$
so that,
$$
X_t\mid X_0=x_0 \sim \mathcal N(A_t x_0+b_t,\Sigma_t) \eqsp.
$$
Thus the conditional law is Gaussian with affine mean in the data point. This is the crucial property. The transition kernel is explicit and when
$\Sigma_t$ is positive definite its conditional score is available in closed
form (see discussion in Section \ref{sec:denoising_score_matching}),
\begin{equation} \label{eq:gaussian_smoothing_score}
\nabla_x\log p_{t|0}(x\mid x_0) = -\Sigma_t^{-1}(x-A_tx_0-b_t) \eqsp.
\end{equation}

Hence affine Gaussian corruptions fit exactly the conditional score matching
identity. The intractable marginal score can be learned by regression against an explicit conditional score.

\paragraph{Scalar affine smoothing.}

Most standard SGMs use the scalar isotropic case
$$
A_t=m_t I_d \eqsp,  \qquad b_t=0 \eqsp, \qquad \Sigma_t=\sigma_t^2 I_d \eqsp.
$$
Equivalently,
\begin{equation}\label{eq:affine-smoothing-coupling}
X_t = m_t X_0+\sigma_t Z \eqsp , \qquad Z \sim \mathcal N(0,I_d),
\qquad Z\perp\!\!\!\perp X_0 \eqsp,
\end{equation}
where $m_0=1$ and $\sigma_0=0$. Conditionally on $X_0=x_0$,
$$
\bar X_t\mid X_0=x_0 \sim \mathcal N(m_t x_0,\sigma_t^2 I_d) \eqsp,
$$
and therefore
$$
\nabla_x\log p_{t|0}(x\mid x_0) = -\frac{x-m_t x_0}{\sigma_t^2} \eqsp.
$$
The marginal law is
$$
\pi_t = (x\mapsto m_t x)_\#\pi_{\rm data} * \mathcal N(0,\sigma_t^2 I_d) \eqsp,
$$
where $(x\mapsto m_tx)_\#\pi_{\rm data}$ is the pushforward of
$\pi_{\rm data}$ by the scaling map $x \mapsto m_tx$, i.e. the law of
$m_tX_0$ when $X_0 \sim \pi_{\rm data}$, and $*$ denotes convolution of
probability measures. Thus $\pi_t$ is obtained by scaling the data distribution and convolving it
with a centered Gaussian. As soon as $\sigma_t > 0$, the law $\pi_t$ is
absolutely continuous with respect to Lebesgue measure and admits a smooth
density $p_t$. In this parametrization, $m_t$ controls the remaining signal amplitude at time $t$, while $\sigma_t$ controls the injected noise. Hence, at time $t=0$ we set $m_0=1$ and $\sigma_0=0$. This scalar affine form covers several standard marginal paths for $t \in (0,T]$.
\begin{itemize}
  \item \textbf{Variance-Exploding (VE)} \citep{song2019generative,song2021scorebased}.
  The data are kept at full scale while the noise level increases:
  $$
  m_t \eqdef 1,
  \qquad
  \sigma_t  \quad \text{nondecreasing} \eqsp.
  $$

  \item \textbf{Variance-Preserving (VP)} \citep{dicksteinthermo,ho2020denoising,song2021scorebased}.
  The signal is progressively attenuated while the overall variance is
  controlled:
  $$
  m_t \eqdef \alpha_t, \qquad \sigma_t \eqdef \sqrt{1-\alpha_t^2}, \qquad \alpha_T \approx 0 \eqsp.
  $$

  \item \textbf{Flow-matching} \citep{lipman2023flowmatching,liu2023rectifiedflow, albergo2023building}. The data and Gaussian noise is linearly interpolated
  $$
  m_t=1-t,
  \qquad
  \sigma_t=t \eqsp, \qquad t \in [0,1] \eqsp.
  $$
  as in flow matching and rectified flow.
\end{itemize}

\paragraph{Extended phase-space Gaussian smoothing.}

The direct affine smoothing is not the only useful Gaussian corruption mechanism for SGMs. One may also introduce auxiliary variables and work on an extended state space. For instance, let
$
U_t=(X_t,V_t)\in\R^{2d} \eqsp,
$
where $X_t\in\R^d$ is the position variable and $V_t\in\R^d$ is an auxiliary
variable. When $V_t$ plays the role of an auxiliary velocity or momentum variable, the extended state $U_t=(X_t,V_t)$ can be viewed as a phase-space variable, in analogy with Hamiltonian mechanics \citep{arnold1989mathematical}. When the dynamics couples position and velocity, we refer to such constructions as \emph{kinetic} models. Such phase-space constructions appear, for
instance, in \emph{Critically-Damped Langevin diffusion} (CLD) models
\citep{dockhorn2022score,conforti_kl, strasman2025wasserstein}. A kinetic affine Gaussian corruption is of the form
$$
U_t = A_t U_0 + b_t+ \Gamma_t Z \eqsp, \qquad Z \sim \mathcal N(0,\Id_{2d}) \eqsp,
\qquad Z\perp\!\!\!\perp U_0 \eqsp ,
$$
with $A_t\in\mathbb R^{2d\times 2d}$, $C_t \eqdef \Gamma_t\Gamma_t^\top\in\mathbb R^{2d\times 2d}$, and $b_t \in \R^{2d}$. If $C_t$ is positive definite, then the conditional law admits a density on $\mathbb R^{2d}$, as in \eqref{eq:gaussian_smoothing_score} but on the extended position-velocity phase space. If $C_t$ is singular, the conditional law remains Gaussian but is supported on a lower-dimensional affine subspace; in that case, it does not admit a classical density with respect to Lebesgue measure on $\R^{2d}$. The question of when a kinetic construction yields a nondegenerate positive-time covariance will be clarified below through the linear SDE realization point of view.

\paragraph{Forward process as solution to a stochastic differential equation (SDE).}
The coupling $A_t X_0 + b_t + \Gamma_t Z$ specifies the marginal law at each time, but not a unique path law. To define a process on path space, a natural way is to view affine Gaussian smoothing for fixed noise levels as the marginal of a linear diffusion processes.

\begin{proposition}[Linear SDEs]
\label{prop:linear_sde_affine_gaussian}
Let $\ora X_0 \sim \pi_{\rm data}$ be independent of a standard Brownian motion $(B_t)_{t\in[0,T]}$ in $\mathbb R^d$. Consider the linear SDE
\begin{equation}\label{eq:linear-sde-realization}
\rmd \ora X_t =(F_t \ora X_t +r_t ) \rmd t + G_t \rmd B_t \eqsp, \qquad \ora X_0 \sim \pi_{\rm data} \eqsp,
\end{equation}
where $F_t\in\mathbb R^{d\times d}$, $r_t\in\mathbb R^d$, and
$G_t\in\mathbb R^{d\times d}$ are deterministic continuous. Let $\Phi_{t,s}$ be the fundamental matrix associated with $F_t$ \citep[Chapter 2]{chicone2006ordinary}, namely the solution of
$$
\partial_t\Phi_{t,s}= F_t\Phi_{t,s} \eqsp,
\qquad
\Phi_{s,s}=\Id_d \eqsp.
$$
Then, conditionally on $\ora X_0=x_0$, the random variable $\ora X_t$ is Gaussian:
$$
\ora X_t\mid \ora X_0=x_0 \sim \mathcal N(M_t x_0+c_t,\Sigma_t) \eqsp,
$$
where $M_t\eqdef \Phi_{t,0}$, $c_t\eqdef \int_0^t \Phi_{t,u}r_u \rmd u$, and
\begin{equation} \label{eq:covariance_forward_linear_SDE}
\Sigma_t \eqdef \int_0^t \Phi_{t,u}G_uG_u^\top\Phi_{t,u}^\top \rmd u \eqsp.
\end{equation}
If $\Sigma_t$ is positive definite, then the conditional density exists and its score is given by
$$
\nabla_x\log p_{t|0}(x\mid x_0) = - \Sigma_t^{-1}(x-M_t x_0-c_t) \eqsp.
$$
\end{proposition}

\begin{proof}
Fix $t\in[0,T]$. We first record the elementary identity \citep[Proposition 2.14]{chicone2006ordinary} 
$$
\partial_u \Phi_{t,u}=-\Phi_{t,u}F_u \eqsp,
\qquad 0\le u\le t \eqsp.
$$
Now, for $u\in[0,t]$, consider the process $Y_u\eqdef \Phi_{t,u}\ora X_u$. Since $u\mapsto \Phi_{t,u}$ is deterministic and continuously differentiable (since $F_t$ is continuous), It\^o's lemma yields
\begin{align*}
\rmd Y_u
&= - \Phi_{t,u} F_u \ora X_u \rmd u + \Phi_{t,u}(F_u\ora X_u+r_u) \rmd u + \Phi_{t,u}G_u \rmd B_u \\
&= \Phi_{t,u} r_u \rmd u + \Phi_{t,u} G_u \rmd B_u \eqsp.  
\end{align*}
Integrating from $0$ to $t$, and using $\Phi_{t,t}=\Id_d$, yields
$$
\ora X_t
= \Phi_{t,0} \ora X_0 + \int_0^t \Phi_{t,u}r_u \rmd u + \int_0^t \Phi_{t,u}G_u \rmd B_u \eqsp.
$$
Conditionally on $\ora X_0=x_0$, the first two terms are deterministic. The last term is a centered Gaussian random vector, its covariance matrix is, by It\^o's
isometry,
$$
\int_0^t \Phi_{t,u} G_u G_u^\top \Phi_{t,u}^\top \rmd u \eqsp.
$$
\end{proof}

\paragraph{Variance-Preserving and Variance-Exploding SDEs.}

As a direct consequence \citep{song2021scorebased} is that the VE and VP noising processes are well described on path space by a diffusion process. Indeed, take
$$
F_t=-\alpha\beta(t) \Id_d \eqsp, \qquad r_t=0 \eqsp, \qquad G_t = \sqrt{2\beta(t)} \Id_d \eqsp,
$$
with $\alpha \ge 0$ and $\beta : [0,T] \to \R_{\ge 0}$ be a continuous noise schedule. Then, \eqref{eq:linear-sde-realization} becomes
\begin{equation} \label{eq:vp_ve_sde}
\rmd \ora X_t = - \alpha \beta(t) \ora X_t \rmd t + \sqrt{2\beta(t)} \rmd B_t \eqsp, \qquad \ora X_0 \sim \pi_{\rm data} \eqsp.
\end{equation}
Writing $\tau_t\eqdef \int_0^t \beta(s) \rmd s$ the fundamental matrix is
$$
\Phi_{t,s} = \exp \left\{ -\alpha \int_s^t \beta(u) \rmd u \right\} \Id_d \eqsp. 
$$
Therefore
$$
\ora X_t\mid \ora X_0=x_0 \sim \mathcal N \left( m_t x_0,\sigma_t^2 \Id_d \right) \eqsp,
$$
with $m_t=\exp(-\alpha\tau_t)$
and
$$
\sigma_t^2
=
\begin{cases}
2\tau_t \eqsp, & \alpha=0 \eqsp, \\[0.5em]
\dfrac{1-\exp(-2\alpha\tau_t)}{\alpha} \eqsp, & \alpha>0 \eqsp.
\end{cases}
$$
Thus $\alpha=0$ gives the Variance-Exploding case, corresponding to a
time-changed Brownian motion, while $\alpha>0$ is the Variance-Preserving case, corresponding to a time-inhomogeneous Ornstein--Uhlenbeck noising process. For precise equivalence between the SDE parametrization as written in \citet{song2021scorebased} and the above see Section C.1 in \citet{strasman2026forgetting}. 

%The function $\beta$ is the noise schedule. It controls the
%rate at which noise is injected into the forward process. The influence of this schedule on SGM generation error is studied in detail in \citet{strasman2025an}.

\begin{figure}[t]
    \centering
    \includegraphics[width=0.8\textwidth]{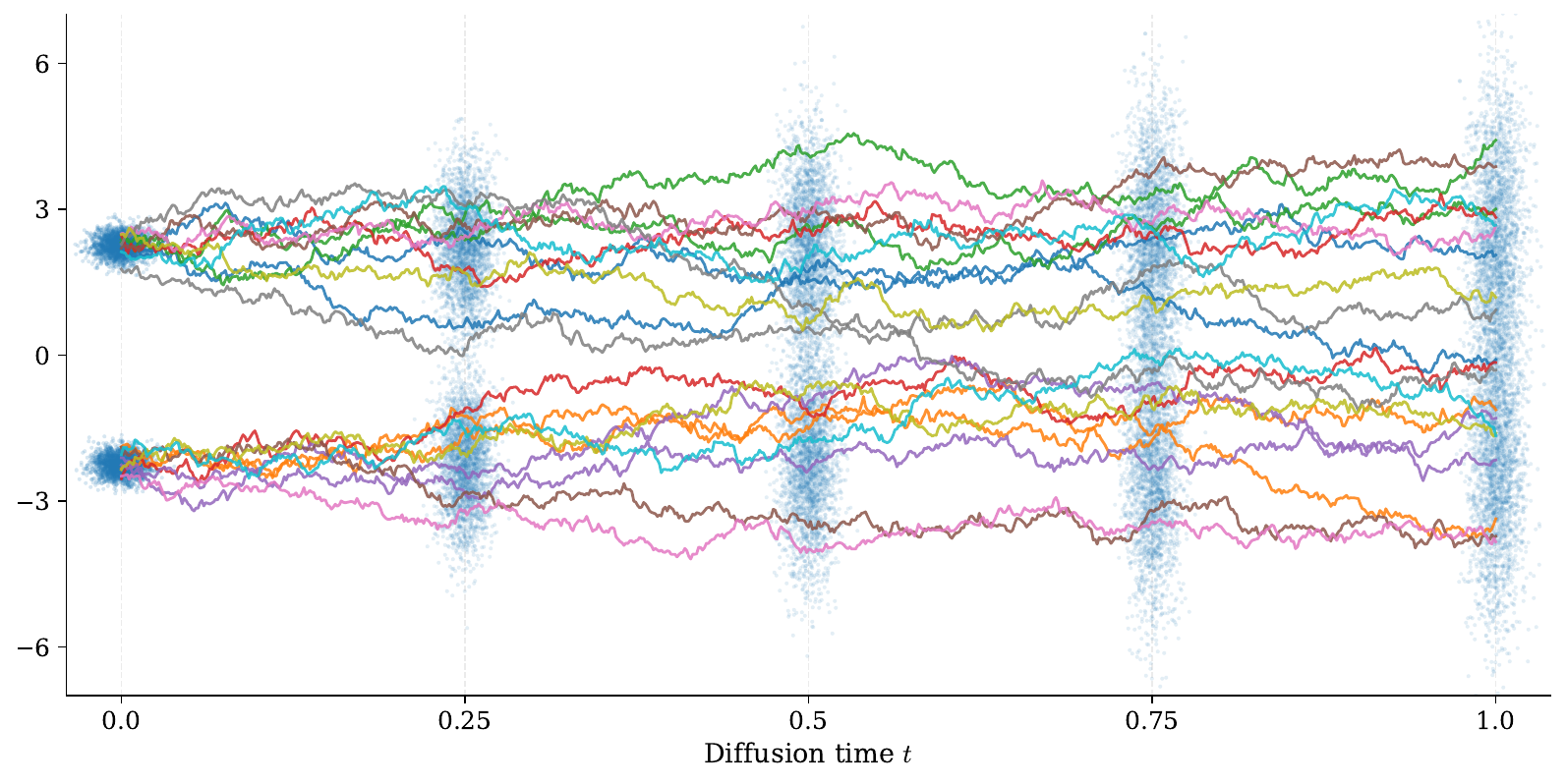}
    \caption{
    Gaussian smoothing marginals and a SDE path-space realization. At the frozen times $t \in \{ 0,0.25,0.5,0.75,1 \}$, the blue vertical clouds are samples from the affine Gaussian marginal
    $X_t\mid X_0=x_0\sim \mathcal N(x_0,2\tau_t)$.
    The colored curves are sample paths of the VE SDE
    $\ora X_t=X_0+B_{2\tau_t}$.
    The horizontal spread of the blue clouds is added only for visualization.
    }
    \label{fig:ve_marginals_with_paths}
\end{figure}

\paragraph{Degenerate noise and kinetic models.}

Proposition~\ref{prop:linear_sde_affine_gaussian} does not require the diffusion
matrix $G_t$ to be invertible. The matrix $G_tG_t^\top$ may be singular, in
which case Brownian noise is injected only in some directions. What matters for
the existence of a Lebesgue density, and hence of a classical conditional score,
is the accumulated covariance $\Sigma_t$ appearing in
\eqref{eq:covariance_forward_linear_SDE}. If $\Sigma_t$ is positive definite
for $t>0$, then the conditional law is a nondegenerate Gaussian law, even
though the instantaneous noise is degenerate.

This situation appears naturally in kinetic or phase-space models, where the
state space is enlarged from \(\mathbb R^d\) to \(\mathbb R^{2d}\). Let
$
\ora U_t=(\ora X_t,\ora V_t)\in\mathbb R^{2d}
$
denote a position--velocity process and consider the linear SDE
\begin{equation} \label{eq:kinetic_sde}
\rmd \ora U_t
=
A\ora U_t\,\rmd t
+
\Gamma \rmd B_t,
\qquad
\ora U_0\sim \pi_{\rm data}\otimes \pi_v,
\end{equation}
where $A \in \R^{2d \times 2d}$, $(B_t)_{t \in [0,T]}$ is a $2d$-dimensional Brownian motion, $\pi_v$ is a chosen distribution independent of $\pi_{\rm data}$, and
$$
\Gamma
=
\begin{pmatrix}
0 & 0\\
0 & \sigma
\end{pmatrix}
\otimes \Id_d \eqsp.
$$
Here the diffusion matrix is singular and so noise is injected only in the velocity variable. Nevertheless, if the drift
matrix $A$ couples position and velocity in a sufficiently nondegenerate way,
then this noise is propagated from velocity to position. In the linear
time-homogeneous case \eqref{eq:kinetic_sde}, the corresponding structural condition to obtain $\Sigma_t$ positive definite is the Kalman rank
condition \citep{metafune2002spectrum,ottobre2015degenerate}
$$
\operatorname{rank} \left[
\Gamma,\ A\Gamma,\ A^2\Gamma,\ldots,A^{2d-1}\Gamma
\right] = 2d \eqsp.
$$
Under this condition, the accumulated covariance is positive definite for every
positive time. The process is then said to be \emph{hypoelliptic} rather than elliptic \citep{hormander1967hypoelliptic}.

\begin{remark}
    In the above, we discussed the Gaussianity of $\ora U_t | \ora U_0$ but in fact if $\ora V_0$ is Gaussian too, independent of $\ora X_0$, then $\ora U_t | \ora X_0$ is also Gaussian. This remark induces important design choices for score matching in kinetic models \citep{dockhorn2022score,strasman2025wasserstein}. 
\end{remark}

\subsection{The time-reversed backward process}

We now turn to the central probabilistic mechanism behind SGMs. The purpose of the forward noising process is to transform the data distribution into a simpler high-noise distribution. In the previous section, we saw that affine Gaussian smoothing can be realized, at the level of path space, by a linear diffusion process (see
Proposition~\ref{prop:linear_sde_affine_gaussian}). The goal of sampling, however, is to run the noising process backward. In fact, starting
from the terminal law of the forward process, one can construct a reverse-time dynamics whose marginals return to the data distribution. The key fact is that, under suitable regularity assumptions, the time reversal of a diffusion is again
a diffusion.

\paragraph{Time reversal of Markov Processes.}

Before identifying the reverse dynamics as a diffusion, let us recall a more elementary fact: the time reversal of a Markov process over a finite time horizon is again Markov. Indeed, the Markov property may be expressed by saying that, conditionally on the present state, the past and the future are independent. This formulation is symmetric with respect to the two time
directions; see, for instance, \citet{doob1953stochastic,dynkin1961markov, meyer1967processus, nagasawa1964time, chung1969reverse} and
the discussion of time symmetries in \citet{leonard2014reciprocal}. 
%A classical reference on time reversal of Markov processes is \citet{nagasawa1964time}. 
Indeed, let $(\ora X_t)_{t\in[0,T]}$ be a Markov process and define the reversed process by
$$
\ola X_t\eqdef \ora X_{T-t} \eqsp.
$$
For $s\in[0,T]$, note $\sigma \left( (\ola X_t)_{t \in [0,s]} \right)= \sigma \left( (\ora X_t)_{t \in [T-s,T]} \right)$ and similarly,
$$ 
\sigma \left( (\ola X_t)_{t \in [s,T]} \right)= \sigma \left( (\ora X_t)_{t \in [0,T-s]} \right) \eqsp .
$$
Since $(\ora X_t)_{t \in [0,T]}$ is Markov, the sigma-fields
$\sigma ( (\ora X_t)_{t \in [0,T-s]} )$ and $\sigma ( (\ora X_t)_{t \in [T-s,T]} )$ are conditionally
independent given $\ora X_{T-s} = \ola X_s$. Hence the past and future of $(\ola X_t)_{t\in[0,T]}$ are
conditionally independent given $\ola X_s$, which is the Markov property for the reversed process.

This explains why the reversed process is Markov. However, being Markov is weaker than being a diffusion. A diffusion process is a continuous-path Markov process whose infinitesimal generator is a second-order differential operator, equivalently, under suitable assumptions, a weak solution of a SDE. The question relevant to Score-based Generative Modeling is therefore stronger: if the forward process is a known diffusion, is the reversed Markov process again a diffusion, and what are its drift and diffusion coefficients?

\paragraph{Time reversal of diffusion processes.} The problem of time reversal for diffusion processes has a long history, motivated in part by questions from statistical mechanics and stochastic
dynamics; see, for example, \citet{schrodinger1931umkehrung,nelson1967dynamical}. In the diffusion
setting, an early influential reference is \citet{anderson1982}, who
derived reverse-time diffusion equations using forward and backward Kolmogorov
equations. A rigorous PDE treatment and refinement was given by \citet{haussmann1986time}. Related integration-by-parts and
Malliavin-calculus approaches were developed by
\citet{millet1989integration}, while entropy and path-space methods appear in
\citet{follmer1985entropy, cattiaux}. Infinite-dimensional extensions were studied, among others, by \citet{follmer1986time} and \citet{millet1989time}. The relevance of reverse-time diffusions for SGMs was emphasized by
\citet{song2021scorebased}.

In the machine-learning literature, the time-reversal result is often introduced by matching the marginal densities through the Fokker--Planck equation of the forward diffusion. This is a useful way to identify the candidate reverse SDE and gives the correct intuition. However, equality of one-time marginals is not a path-space statement, as it does
not by itself identify the law of the process on $C([0,T],\mathbb R^d)$. To
identify the actual time-reversed process $(\ola X_{t})_{t\in[0,T]}$
one needs an argument at the level of paths. The martingale-problem formulation provides such an argument: it characterizes the reversed process by testing its generator against smooth functions and proving the associated martingale property \citep[Chapter 4]{ethier1986markov}.

The result below is a simplified version of the classical time-reversal theorem of \citet{haussmann1986time}, specialized to deterministic diffusion matrices. The proof sketch is included to highlight the main mechanisms relevant to SGMs.

\begin{proposition} \label{prop:time_reversal_deterministic_diffusion}
Let the forward process $(\ora X_t)_{t\in[0,T]}$ admit a unique strong solution and solve
$$
\rmd \ora X_t = b(t,\ora X_t) \rmd t + G_t \rmd B_t \eqsp,
$$
where \(G_t\in\mathbb R^{d\times d}\) is deterministic, possibly singular, and set $a_t\eqdef G_tG_t^\top$.

Assume that the following regularity assumptions
hold:
\begin{enumerate}[label=(\roman*)]
    \item For every $t\in(0,T]$,  $\ora X_t$ admits a strictly positive density $p_t\in C^{\infty}(\mathbb R^d)$.

    \item For every $0\le u<r \leq T$ and every
    $\varphi\in C_c^\infty(\mathbb R^d)$, the function
    $$
    h(v,x) = \mathbb E\left[\varphi(\ora X_r^{v,x})\right] \eqsp,
    \qquad v\in[u,r] \eqsp,
    $$
    belongs to $C^{1,2}([u,r)\times\mathbb R^d)$,
satisfies $h(r,\cdot)=\varphi$, extends continuously to $v=r$, and solves the backward Kolmogorov equation.
\end{enumerate}
Then, for $0\le t<T$, the time-reversed process $\ola X_t\eqdef \ora X_{T-t}$ solves the martingale problem, for $f\in C_c^\infty(\mathbb R^d)$, associated with
\begin{equation} \label{eq:candidate_back_gen}
\ola L_t f(x)
=
\left\langle -b(T-t,x) + a_{T-t}\nabla\log p_{T-t}(x), \nabla f(x) \right\rangle
+
\tfrac12 \operatorname{Tr} \left( a_{T-t}\nabla^2 f(x) \right) \eqsp.
\end{equation}
Consequently, the reversed process can be represented weakly as
$$
\rmd \ola X_t
= \left[ -b(T-t,\ola X_t) + a_{T-t}\nabla\log p_{T-t}(\ola X_t) \right]\rmd t + G_{T-t} \rmd \ola B_t \eqsp.
$$
where $(\ola B_t)_{t\in[0,T]}$ is a Brownian motion in $\R^d$ driving the reverse
dynamics.
\end{proposition}

\begin{remark}
The Brownian motion $(\ola B_t)_{t\in [0,T]}$ appearing in the reverse SDE should not be identified pathwise with the Brownian motion $(B_t)_{t\in [0,T]}$ driving the forward SDE.
It is a Brownian motion for the reversed dynamics. In generative modeling, however, one usually simulates a new solution of the reverse SDE rather than reversing a given forward trajectory. In this sense, the above result is only used in a weak sense, since only the distribution of the generated path is considered. It is therefore harmless to
denote the driving Brownian motion again by \((B_t)_{t\in[0,T]}\), provided no
pathwise identification with the Brownian motion of the forward process is
intended.
\end{remark}

\begin{remark}
The assumptions above are natural for the linear Gaussian processes introduced
in Proposition~\ref{prop:linear_sde_affine_gaussian}. Consider
\begin{equation} \label{eq:forward_SDE_rebelotte}
\rmd \ora X_t = (F_t \ora X_t + r_t) \rmd t + G_t \rmd B_t \eqsp,
\end{equation}
with deterministic continuous coefficients. For $0\le u<r\le T$ , recall that $\ora X_r\mid \ora X_u$ is Gaussian with covariance matrix
$$
\Sigma_{r|u} = \int_u^r \Phi_{r,v} G_v G_v^\top \Phi_{r,v}^\top \rmd v \eqsp .
$$
If $\Sigma_{r|u}\succ0 $  then the transition kernel associated with \eqref{eq:forward_SDE_rebelotte}  admits a smooth strictly positive Gaussian density and solves the backward Kolmogorov equation.
\end{remark}

\begin{proof}[Sketch of proof]
The candidate generator of the reversed process is given in \eqref{eq:candidate_back_gen}. For $f\in C_c^\infty(\mathbb R^d)$, we want to show that
$$
M_t^f \eqdef f(\ola X_t)-f(\ola X_0) - \int_0^t \ola L_s f(\ola X_s)\,\rmd s
$$
is a martingale with respect to the reversed filtration. Fix $0\le s\le t\le T$ and let $\mathcal G_t\eqdef \sigma(\ola X_r:0\le r\le t)$. To prove that
$M_t^f$ is a $(\mathcal G_t)$-martingale, it is enough to show that
$$
\mathbb E \left[M_t^f-M_s^f\mid \mathcal G_s\right]=0.
$$
Since the reversed process is Markov, the conditional expectation of the future
increment $M_t^f-M_s^f$ given $\mathcal G_s$ is a function of the present
state $\ola X_s$. Thus it suffices to prove, for every
$\varphi\in C_c^\infty(\mathbb R^d)$,
$$
\mathbb E \left[ \varphi(\ola X_s) \left(M_t^f-M_s^f\right) \right] =0 \eqsp.
$$
The restriction to smooth compactly supported \(\varphi\) is only for the proof
sketch; one then extends the identity to bounded measurable functions by a
monotone-class argument. Recall that the forward generator is given, for
$f\in C_c^\infty(\mathbb R^d)$, by
$$
L_r f(x) = \left\langle b(r,x),\nabla f(x)\right\rangle + \frac12 \operatorname{Tr} \left( a_r\nabla^2 f(x) \right) \eqsp.
$$
Set $\tau_s\eqdef T-s$ and $\tau_t\eqdef T-t$. Let $\ora X^{u,x}$ denote the solution of the forward SDE started from $x$
at time $u$. For $u\in[\tau_t,\tau_s]$, define
\begin{equation} \label{eq:h}
h(u,x)
\eqdef
\mathbb E\!\left[
\varphi(\ora X_{\tau_s}^{u,x})
\right] \eqsp.
\end{equation}
Under the assumed regularity assumption $h\in
C^{1,2}([\tau_t,\tau_s)\times\mathbb R^d)$, and $h$ solves the backward Kolmogorov
equation
\begin{equation}\label{eq:kolmogorov_back}
\partial_u h(u,x) + L_u\big(h(u,\cdot)\big)(x) = 0 \eqsp,
\qquad
h(\tau_s,\cdot) = \varphi \eqsp .
\end{equation}
Applying It\^o's product law to $f(\ora X_u)h(u,\ora X_u)$ yields
\begin{align*}
\rmd \left( f (\ora X_u ) h ( u,\ora X_u ) \right) & = h ( u, \ora X_u ) \rmd f ( \ora X_u ) + f ( \ora X_u ) \rmd h ( u,\ora X_u ) \\
 & \qquad \qquad + \left\langle a_u\nabla f ( \ora X_u ), \nabla h ( u,\ora X_u ) \right\rangle \rmd u \eqsp.
\end{align*}
Using \eqref{eq:kolmogorov_back}, $\rmd h(u, \ora X_u) = \langle \nabla h(u , \ora X_u ), G_u \rmd B_u \rangle$ and the above simplifies to 
\begin{align*}
 & \rmd \left( f ( \ora X_u ) h (u,\ora X_u ) \right)  = h ( u, \ora X_u ) L_u f ( \ora X_u ) \rmd u +  h ( u, \ora X_u ) \left\langle \nabla f (  \ora X_u ), G_u \rmd B_u \right\rangle \\
& \qquad \qquad \qquad  + f ( \ora X_u ) \left\langle \nabla h ( u , \ora X_u ), G_u \rmd B_u \right\rangle + \left\langle a_u \nabla f ( \ora X_u ) , \nabla h ( u , \ora X_u ) \right\rangle \rmd u  \eqsp.
\end{align*}
Taking expectation,
\begin{multline*}
\mathbb E \left[ f(\ora X_{\tau_s}) h(\tau_s,\ora X_{\tau_s}) \right] - \mathbb E \left[ f(\ora X_{\tau_t})h(\tau_t,\ora X_{\tau_t}) \right] \\
= \int_{\tau_t}^{\tau_s} \mathbb E \left[ h(r,\ora X_r)L_r f(\ora X_r) \right] \rmd r + \int_{\tau_t}^{\tau_s} \mathbb E \left[ \left\langle a_r\nabla f(\ora X_r), \nabla h(r,\ora X_r) \right\rangle \right] \rmd r \eqsp.
\end{multline*}
Using that $h(\tau_s, x) = \varphi(x)$ and using \eqref{eq:h},
\begin{align*}
  h (\tau_t, \ora X_{\tau_t}) 
  &= \E \left[ \varphi (\ora X_{\tau_s} ) \middle| \sigma \left( \ora X_{\tau_t} \right)  \right]  
\end{align*}
and therefore since $f( \ora X_{\tau_t})$ is $\sigma( \ora X_{\tau_t})$-mesurable,
$$
\E \left[ f( \ora X_{\tau_t}) h (\tau_t, \ora X_{\tau_t}) \right] = \E \left[ f( \ora X_{\tau_t})  \E \left[ \varphi (\ora X_{\tau_s} ) \middle| \sigma \left( \ora X_{\tau_t} \right)  \right]\right] = \E \left[ f( \ora X_{\tau_t})  \varphi (\ora X_{\tau_s} ) \right] \eqsp.
$$
Hence,
\begin{align*}
&\mathbb E \left[ \varphi(\ora X_{\tau_s}) \left( f(\ora X_{\tau_s})-f(\ora X_{\tau_t}) \right) \right] \\
&\quad= \int_{\tau_t}^{\tau_s} \mathbb E \left[ h(r,\ora X_r)L_r f(\ora X_r) \right] \rmd r + \int_{\tau_t}^{\tau_s} \mathbb E \left[ \left\langle a_r\nabla f(\ora X_r), \nabla h(r,\ora X_r) \right\rangle \right] \rmd r \eqsp.
\end{align*}
We now treat the second term by integration by parts. Since
$\ora X_r$ has density $p_r$ and $a_r$ does not depend on the space
variable, we have
\begin{align*}
\mathbb E \left[
\left\langle a_r \nabla f(\ora X_r), \nabla h(r,\ora X_r) \right\rangle \right]  
& = \int_{\mathbb R^d} \left\langle a_r\nabla f(x), \nabla h(r,x) \right\rangle p_r(x) \rmd x \\
& = -\int_{\mathbb R^d} h(r,x) \operatorname{div} \left( p_r(x)a_r\nabla f(x) \right) \rmd x \\
& = -\mathbb E \Bigg[ h(r,\ora X_r) \Big\{ \operatorname{Tr} \left(a_r\nabla^2 f(\ora X_r \right) \\
& \qquad \qquad+ \left\langle a_r \nabla \log p_r(\ora X_r), \nabla f(\ora X_r) \right\rangle \Big\} \Bigg] \eqsp.
\end{align*}
Therefore,
\begin{align*}
&\mathbb E \left[ \varphi(\ora X_{\tau_s}) \left( f(\ora X_{\tau_t})-f(\ora X_{\tau_s}) \right) \right] \\
& = \int_{\tau_t}^{\tau_s} \mathbb E \Bigg[ h(r,\ora X_r) \bigg\{ \left\langle -b(r,\ora X_r)
+ a_r\nabla\log p_r(\ora X_r), \nabla f(\ora X_r) \right\rangle \\
& \qquad \qquad \qquad \qquad \qquad \qquad \qquad \qquad \qquad \qquad \qquad \qquad  + \tfrac12 \operatorname{Tr} \left( a_r\nabla^2 f(\ora X_r) \right) \bigg\} \Bigg] \rmd r \\
& = \int_{\tau_t}^{\tau_s} \mathbb E \Bigg[ \varphi(\ora X_{\tau_s})
\bigg\{ \left\langle -b(r,\ora X_r) + a_r \nabla \log p_r(\ora X_r), \nabla f(\ora X_r)\right\rangle \\
& \qquad \qquad \qquad \qquad \qquad \qquad \qquad \qquad \qquad \qquad \qquad \qquad  + \tfrac12 \operatorname{Tr} \left( a_r\nabla^2 f(\ora X_r) \right) \bigg\} \Bigg] \rmd r \eqsp,
\end{align*}
where we used \eqref{eq:h} in the second equality. By change of variable $r = T-u$ and using \eqref{eq:candidate_back_gen}
$$
\mathbb E \left[ \varphi(\ola X_s) \left( f(\ola X_t)-f(\ola X_s) \right) \right]
= \mathbb E \left[ \varphi(\ola X_s) \int_s^t \ola L_u f(\ola X_u) \rmd u \right] \eqsp.
$$
It follows that, for any $f\in C_c^\infty(\mathbb R^d)$,
$$
\mathbb E \left[ \varphi(\ola X_s) \left( M_t^f - M_s^f \right) \right] =0 \eqsp,
$$
and since the identity holds for every $\varphi\in C_c^\infty(\mathbb R^d)$, a
standard monotone-class argument extends it to bounded measurable functions of
$\ola X_s$. Hence, the reversed process solves the martingale problem associate with $\ola L_t$ and therefore is a weak solution of
$$
\rmd \ola X_t = \left[ -b(T-t,\ola X_t) + a_{T-t}\nabla\log p_{T-t}(\ola X_t) \right]\rmd t + G_{T-t} \rmd \ola B_t \eqsp.
$$
\end{proof}

\paragraph{Time-reversal and score function.} 
As a consequence, in SGMs, the generative model is not obtained by merely choosing a curve of densities, as in annealing procedures. It is obtained by reversing the path law of a diffusion process. The score appears because reversing a diffusion requires an integration by parts against the forward marginal density $p_t$. In this sense, the term
$$
G_t G_t^\top \nabla\log p_t \eqsp,
$$
can be viewed as a stochastic correction of forward process. This correction has a natural geometric interpretation. The score points in the direction of steepest increase (adjusted by the covariance structure of the forward noise) of the log marginal density.

\begin{remark}[Brownian reversal versus frozen Langevin dynamics]
Let
$$
\rmd \ora X_t= \sqrt 2 \rmd B_t \eqsp.
$$
By the time-reversal formula, the reversed process satisfies
\begin{equation} \label{eq:reverse_heat_flow}
\rmd \ola X_t = 2 \nabla\log p_{T-t}(\ola X_t) \rmd t + \sqrt 2 \rmd B_t \eqsp .
\end{equation}
Therefore, this dynamics is formally close to the frozen Langevin dynamics
defined in \eqref{eq:langevin_multi_step} and given by
$$
\rmd Y_t =  \nabla\log p_{T-t}(Y_t) \rmd t + \sqrt 2 \rmd B_t \eqsp .
$$
Nevertheless, the factor $2$ in front of the score terms is essential. Indeed, by Fokker-Planck equation, the forward Brownian marginals solve $\partial_t p_t=\Delta p_t$, so that,
$$
\partial_t p_{T-t}= -\Delta p_{T-t} \eqsp.
$$
Now, for $c \ge 0$, consider a dynamics of the form
$$
\rmd Z_t = c \nabla\log p_{T-t}(Z_t) \rmd t + \sqrt{2} \rmd B_t \eqsp.
$$
Denote by $\rho_t$ the density of $Z_t$. Its Fokker--Planck equation is
$$
\partial_t\rho_t
= -\operatorname{div} \left( c \rho_t \nabla \log p_{T-t} \right) + \Delta\rho_t \eqsp.
$$
If $\rho_t=p_{T-t}$, then
$$
-\operatorname{div} \left( c p_{T-t} \nabla \log p_{T-t} \right) + \Delta p_{T-t} 
= -c \Delta p_{T-t} + \Delta p_{T-t}
= (1-c) \Delta p_{T-t} \eqsp.
$$
Therefore, to recover the reversed heat flow \eqref{eq:reverse_heat_flow} one must choose $c=2$, while when $c=1$, the right-hand side is zero. This illustrates again the subtle distinction between local equilibration and time reversal. The drift $\nabla\log p_t$ makes $p_t$ invariant when time is frozen, whereas the drift $2\nabla\log p_{T-t}$ is the correction dictated by the forward Brownian path law and transports the marginals backward along the heat flow.

This is particularly visible in the two-cluster toy model \eqref{eq:pi_data_two_clusters}. Recall that
$$
\pi_{\rm data} \eqdef \frac12 \delta_{-a} + \frac12 \delta_{+a} \eqsp., \qquad a \in \R \eqsp.
$$
Under the Brownian noising process $\rmd \ora X_s=\sqrt2 \rmd B_s$ with $\ora X_0\sim\pi_{\rm data}$, the marginal law at time $s>0$ has density
$$
p_s(x) = \frac12\varphi_{2s}(x-a) + \frac12\varphi_{2s}(x+a) \eqsp,
$$
where \(\varphi_v\) denotes the density of \(\mathcal N(0,v)\). Its score is explicit:
$$
\partial_x\log p_s(x) = -\frac{x}{2s} + \frac{a}{2s} \tanh \left(\frac{ax}{2s}\right) \eqsp.
$$
We compare two reverse-time dynamics:
\begin{itemize}
    \item Time-reversed process:
    $$ 
    \rmd \ola X_t = 2\partial_x\log p_{T-t}(\ola X_t) \rmd t + \sqrt2 \rmd B_t \eqsp;
    $$
    \item Frozen Langevin:
    $$
    \rmd Y_t = \partial_x\log p_{T-t}(Y_t) \rmd t + \sqrt2 \rmd B_t \eqsp.
    $$
\end{itemize}
Both dynamics use the exact analytical score and are initialized at $p_T$ so there is no learning error, no initialization error and no approximation error (errors that will be formally defined in the next section). Since $p_0$ is a pair of Dirac masses and does not admit a Lebesgue density, we stop the reverse dynamics at a small positive time $s_{\min}>0$, and compare
the final laws at time $T-s_{\min}$ of both dynamics to $p_{s_{\min}}$. Results are shown in Figure \ref{fig:ve_two_cluster_final_density_comparison}
\begin{figure}[t]
    \centering
    \includegraphics[width=0.6\textwidth]{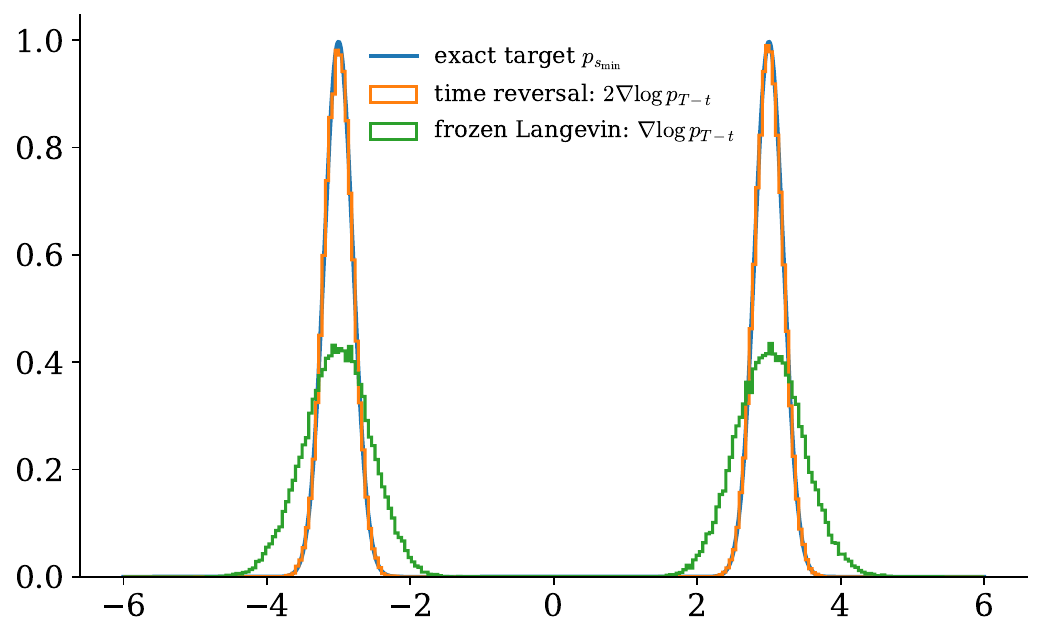}
    \caption{
    Final density comparison in the 1-$d$ VE two-cluster experiment. The blue curve is the exact target density $p_{s_{\min}}$. The orange curve is the empirical final law of the true Brownian time-reversal dynamics. The green curve is the empirical final law of the frozen Langevin dynamics.
    }
    \label{fig:ve_two_cluster_final_density_comparison}
\end{figure}

\end{remark}

\section{Score-based generative models}

\subsection{SGMs in practice}

\paragraph{Forward-backward dynamics.} The previous section showed that, once a forward noising process
\((\ora X_t)_{t\in[0,T]}\) has been specified, its reverse-time dynamics is
determined by the forward coefficients together with the scores of its time
marginals $\nabla\log p_t$. In the linear Gaussian case, denoising score matching provides a tractable way to learn this vector field from samples, by exploiting the explicit conditional scores of the forward
Gaussian corruption kernels. In particular, let $(\ora X_t)_{t\in[0,T]}$ be a forward noising process solution to
\begin{equation} \label{eq:forward_chap_SGM}
\rmd \ora X_t = b(t,\ora X_t)\,\rmd t + G_t \rmd B_t,
\qquad \ora X_0\sim \pi_{\rm data} \eqsp,
\end{equation}
where $G_t$ is deterministic. Denote by $\pi_t\eqdef \mathcal L(\ora X_t)$
its time marginals and let $p_t$ be its density. Suppose that $b(t,\ora X_t)$ is such that $\ora X_t | \ora X_0$ is Gaussian. By the time-reversal formula, $(\ola X_t)_{t \in [0,T]} \eqdef (\ora X_{T-t})_{t \in [0,T]} $
solves
\begin{equation}\label{eq:sgm-backward-exact}
\rmd \ola X_t = \left[ -b(T-t,\ola X_t) + a_{T-t} \nabla \log p_{T-t}(\ola X_t) \right]\rmd t + G_{T-t} \rmd  B_t, \quad \ola X_0\sim \pi_T \eqsp,
\end{equation}
with $a_{T-t} =G_{T-t} G_{T-t}^\top$.

\paragraph{Generative model.} The time-reversal formula gives the recipe of SGMs: choose a
forward process that transforms the data distribution into a simple
distribution, learn the intermediate score functions, and simulate the
corresponding reverse-time dynamics. In the idealized setting where the terminal
law of the forward process $\pi_T$ is sampled exactly, the score function is known exactly, and
the reverse SDE is simulated without time-discretization error, the final law of
the reverse process is precisely $\pi_{\rm data}$. In practice, each of these
idealizations is replaced by an approximation. The study of convergence and
stability of SGMs amounts to understanding how initialization error,
score approximation error, and time-discretization error propagate through the
reverse dynamics and affect the final generated distribution.

% For notational clarity, we focus on the two most common linear Gaussian noising processes: the variance-exploding case, corresponding to a time-changed Brownian motion, and the variance-preserving case, corresponding to a time-inhomogeneous Ornstein--Uhlenbeck process. Both can be written in the unified form
% \begin{equation}\label{eq:sgm-forward}
% \rmd \ora X_t
% =
% -\alpha\beta(t)\ora X_t\,\rmd t
% +
% \sqrt{2\beta(t)}\,\rmd B_t,
% \qquad
% \ora X_0\sim\pi_{\rm data},
% \end{equation}
% where $\alpha\ge 0$ and $\beta:[0,T]\to\mathbb R_+$ is a noise schedule. The
% case $\alpha=0$ gives the variance-exploding setting, while $\alpha>0$ gives an
% Ornstein--Uhlenbeck-type noising process. Its time reversed process $( \ola X_t)_{t \in [0,T]}$ is therefore solution to
% \begin{equation}\label{eq:sgm-backward-exact}
% \rmd \ola X_t
% =
% \bar\beta(t)
% \left(
% \alpha \ola X_t
% +
% 2 \nabla \log p_{T-t}(\ola X_t)
% \right)\rmd t
% +
% \sqrt{2\bar\beta(t)}\,\rmd \ola B_t,
% \qquad
% \ola X_0\sim \pi_T,
% \end{equation}
% Although we restrict to this class for simplicity, the same conceptual decomposition applies more generally to time-reversed diffusion-based generative models.

\paragraph{Semi-group notations.}
For $0\le s\le t\le T$, we introduce the time-dependent transition kernels associated with \eqref{eq:sgm-backward-exact} defined for $A \in \mathcal{B}(\R^d)$ and $x \in \R^d$ as
$$
Q_{s,t}(x,A)
\eqdef
\mathbb P\!\left(\ola X_t^{s,x}\in A\right),
$$
where $\ola X^{s,x}$ denotes the solution of \eqref{eq:sgm-backward-exact} started from $x$ at time $s$. Equivalently, for a bounded measurable test function $f$,
$$
Q_{s,t}f(x)
\eqdef
\mathbb E\!\left[f(\ola X_t^{s,x})\right].
$$
% With this notation, the time reversal
% identity writes as
% $$
% \pi_T Q_{0,T}=\pi_{\rm data} \eqsp.
% $$

\paragraph{Initialization or mixing error.}

The exact reverse dynamics is defined with initial law $\pi_T$. Indeed, the time-reversal identity reads as
$$
\pi_T Q_{0,T} = \pi_{\rm data} \eqsp.
$$
Thus, if one could sample exactly from $\pi_T$ and simulate the exact reverse dynamics, the final generated law would be the data distribution. However, $\pi_T$ is data-dependent, 
$$
p_T(x) = \int_{\mathbb R^d} q_{T|0} (x|x_0) \pi_{\rm data}(\rmd x_0) 
$$
where $q_{T|0}(\cdot | x_0)$ is the Gaussian transition density of the forward noising process \eqref{eq:forward_chap_SGM} from time $0$ to time $T$ (under the assumption that $b(t,x)$ is affine in $x$). Hence, drawing exactly from $\pi_T$ would require drawing
$X_0\sim\pi_{\rm data}$, whereas the whole point of generative models is to produce samples without access to data points. 
That is why, in practice, this initial distribution is replaced by a tractable reference distribution $\pi_\infty$. The nature of the reference law $\pi_\infty$ depends on the choice of the forward process. Let $Z\sim\mathcal{N} (0,\Id_d)$, $Z \indep \ora X_0$ and
$$
\tau_T \eqdef \int_0^T\beta(t) \rmd t \eqsp.
$$
The following examples illustrate different choices of $\pi_\infty$.
\begin{itemize}
    \item \textbf{Variance-exploding (VE)} (Figure \ref{fig:smoothing_ve}), $\alpha=0$, in \eqref{eq:vp_ve_sde}. Then
    $$
    \ora X_T \overset{\mathcal{L}}{=} \ora X_0+\sqrt{2\tau_T} Z 
     \eqsp.
    $$
    There is no invariant distribution in this case: the variance keeps increasing. One therefore uses a high-variance Gaussian
    reference law, typically
    $$\pi_\infty=\mathcal N(0,2\tau_T\Id_d) \eqsp.$$

    \item \textbf{Variance-preserving case (VP)} (Figure \ref{fig:smoothing_vp}), $\alpha >0 $, in \eqref{eq:vp_ve_sde}. Then
    $$
    \ora X_T \overset{\mathcal{L}}{=} \rme^{-\alpha\tau_T} \ora X_0 + \alpha^{-1/2} \sqrt{1-\rme^{-2\alpha\tau_T}} Z \eqsp,  \qquad
    \tau_T\eqdef\int_0^T\beta(t) \rmd t \eqsp.
    $$
    As $\tau_T\to\infty$, the dependence on $\ora X_0$ vanishes and $\pi_T$ converges to $\pi_\infty$. In this setting, the initialization error is genuinely a mixing or forgetting error. As the process is ergodic, one typically chooses the invariant distribution as a reference
    $$\pi_\infty=\mathcal N(0,\alpha^{-1}\Id_d) \eqsp.$$ 

    % \item \textbf{Kinetic models.}
    % In kinetic extensions, the forward noising process convergence will depend on the spectrum of the matrix drift. A common choice, is to consider the harmonic oscillator drift in \eqref{eq:kinetic_sde}
    % $$
    % A = \begin{pmatrix}
    %     0 & 1 \\ -1 & -2
    % \end{pmatrix} \otimes \Id_d \quad \text{and} \quad \Sigma = \begin{pmatrix}
    %     0 & 0 \\ 0 & 2
    % \end{pmatrix} \otimes \Id_d
    % $$
    % so that the process solution to \eqref{eq:kinetic_sde} is ergodic with stationary distribution 
    % $$\pi_{\infty} \sim \mathcal{N}(0_{2d \times 2d}, \Id_{2d}) \eqsp,$$
    % which is the natural reference distribution for such model.
    \item \textbf{Kinetic models} (Figures \ref{fig:smoothing_kinetic_x} and \ref{fig:smoothing_kinetic_v}).
    In kinetic extensions, the forward process evolves on position-velocity phase space. Let $(\ora U_t )_{t \in [0,T]}=(\ora X_t,\ora V_t)_{t \in [0,T]} \in\mathbb R^{2d}$ and consider the time-changed critically damped linear dynamics
    \begin{equation} \label{eq:kou_intro}
    \rmd \ora U_t
    = \beta(t)A\ora U_t \rmd t + \sqrt{\beta(t)} \Sigma \rmd B_t \eqsp,
    \qquad \ora U_0\sim \pi_{\rm data}\otimes \mathcal N(0, \sigma_v^2\Id_d) \eqsp,
    \end{equation}
    with
    $$
    A=
    \begin{pmatrix}
        0 & 1\\
        -1 & -2
    \end{pmatrix}
    \otimes \Id_d \eqsp,
    \qquad \Sigma=
    \begin{pmatrix}
        0 & 0\\
        0 & 2 \sigma_\infty
    \end{pmatrix}
    \otimes \Id_d \eqsp.
    $$
    Then,
    $$
    \ora U_T \overset{\mathcal{L}}{=} \rme^{\tau_T A}\ora U_0 + \sigma_\infty
    \left( \Id_{2d} - \rme^{\tau_T A}\rme^{\tau_T A^\top} \right)^{1/2}Z_{2d} \eqsp,
    $$
    where $Z_{2d}\sim\mathcal N(0,\Id_{2d})$ and $Z_{2d} \indep \ora U_0$. The matrix $A$ is stable, so as $\tau_T\to\infty$,
    $\pi_T$ converges to the Gaussian distribution on $\R^{2d}$
    $$
    \pi_\infty=\mathcal N(0,\sigma_\infty^2\Id_{2d}) \eqsp.
    $$
\end{itemize}

\begin{figure}[p]
    \centering

    \begin{subfigure}{0.94\textwidth}
        \centering
        \includegraphics[width=\textwidth]{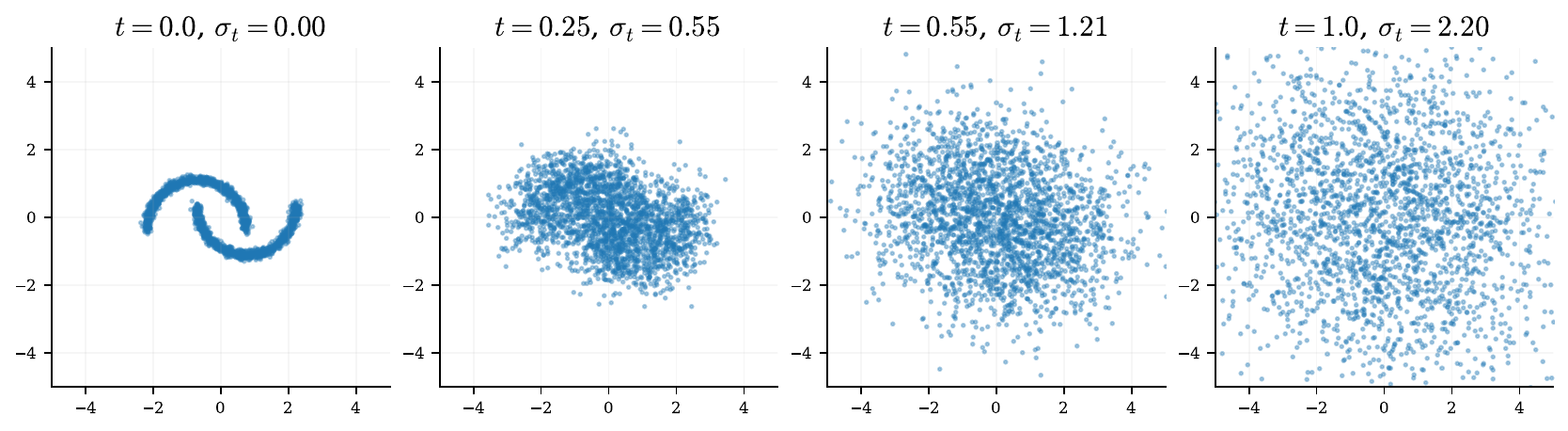}
        \caption{VE: $\ora X_t \overset{\mathcal L}{=} \ora X_0+\sigma_t Z$.}
        \label{fig:smoothing_ve}
    \end{subfigure}

    \vspace{0.6em}

    \begin{subfigure}{0.94\textwidth}
        \centering
        \includegraphics[width=\textwidth]{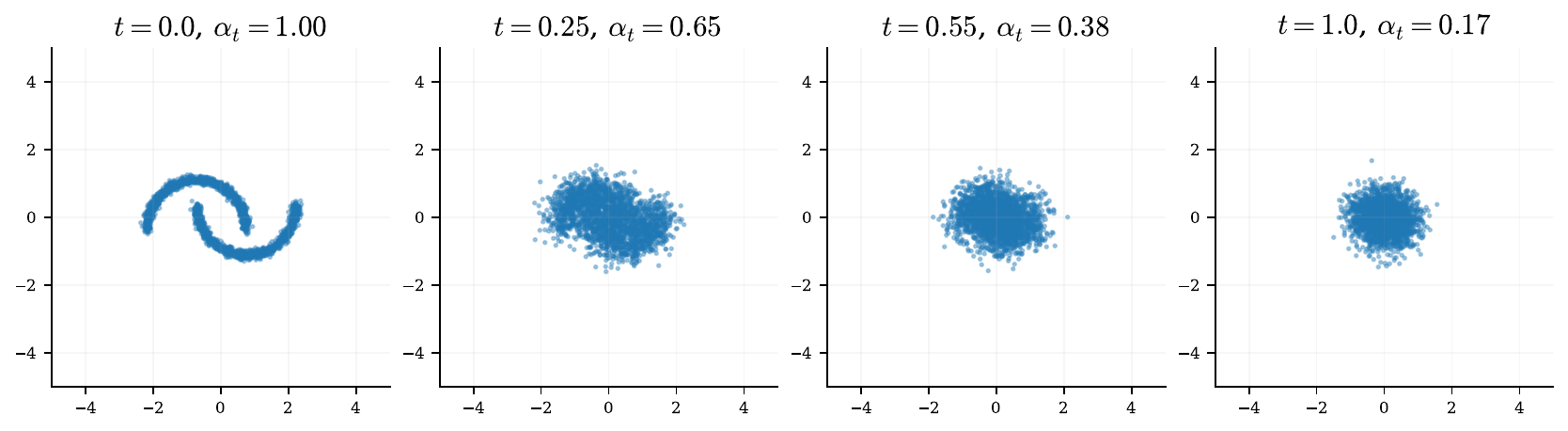}
        \caption{VP: $\ora X_t \overset{\mathcal L}{=} \alpha_t\ora X_0+0.4 \sqrt{1-\alpha_t^2}Z$ with $\alpha_t = \rme^{-\tau_t/2}$.}
        \label{fig:smoothing_vp}
    \end{subfigure}

    \vspace{0.6em}

    \begin{subfigure}{0.94\textwidth}
        \centering
        \includegraphics[width=\textwidth]{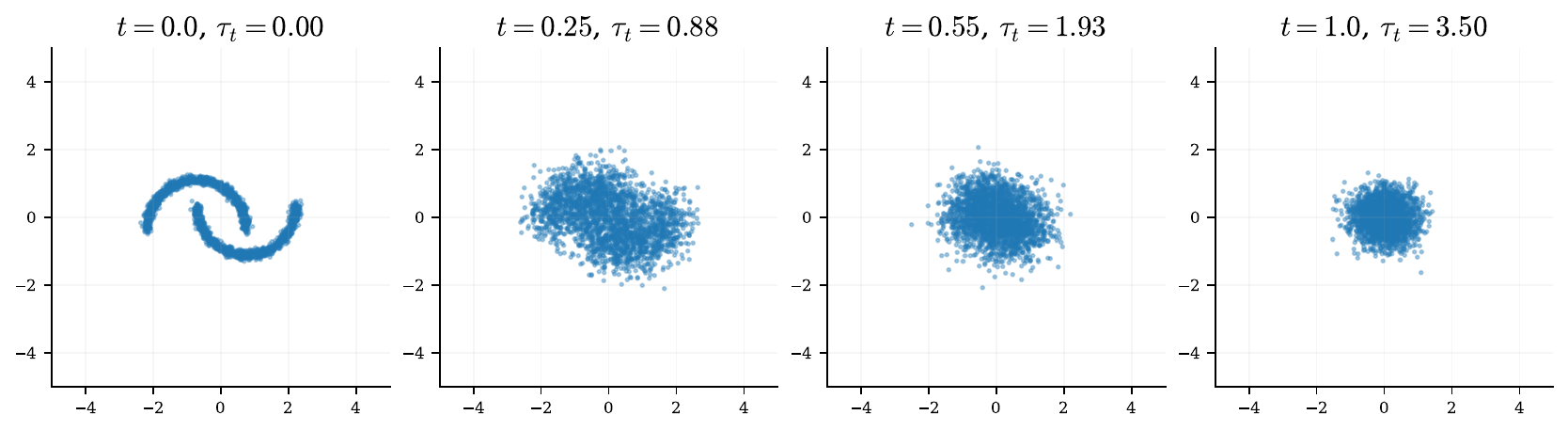}
        \caption{Kinetic (position marginal): $\ora X_t$.}
        \label{fig:smoothing_kinetic_x}
    \end{subfigure}

    \vspace{0.6em}

    \begin{subfigure}{0.94\textwidth}
        \centering
        \includegraphics[width=\textwidth]{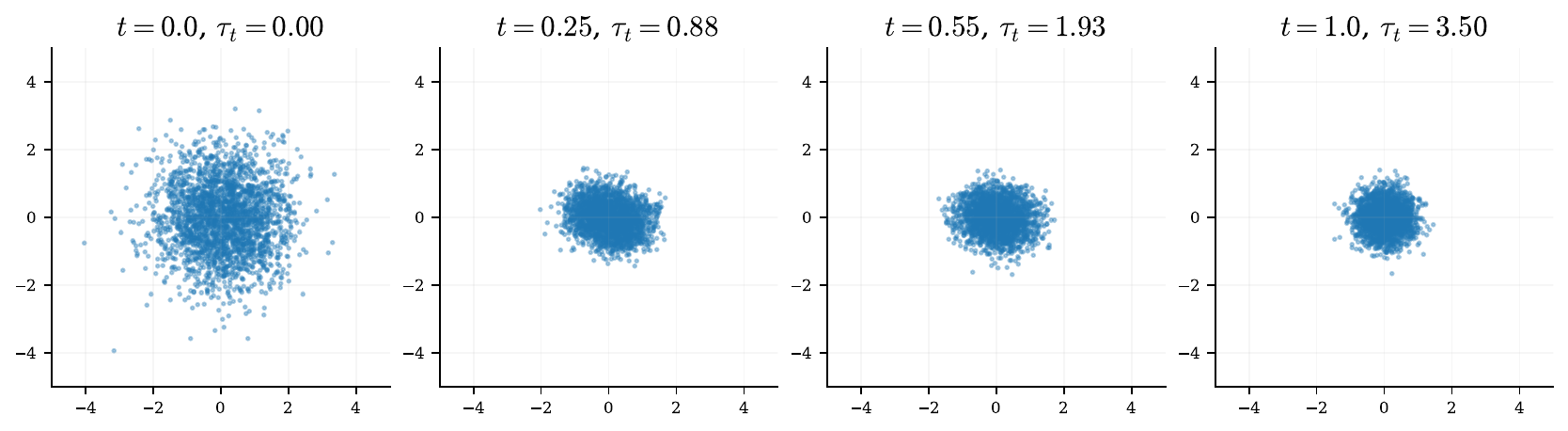}
        \caption{Kinetic(velocity marginal): $\ora V_t$ with $\ora V_0 \sim \mathcal{N}(0_2,\Id_2)$}
        \label{fig:smoothing_kinetic_v}
    \end{subfigure}
    \caption{
    Comparison of Gaussian mixing mechanisms applied to the same two-dimensional data distribution. In the VE case, the noise level $\sigma_t$ increases
    with time. In the VP case, $\alpha_t=\rme^{-\tau_t/2}$, where
    $\tau_t=\int_0^t\beta(s) \rmd s$, and the stationary standard deviation is
    $\sigma_\infty=0.4$. The kinetic model is defined on phase space
    $\ora U_t=(\ora X_t,\ora V_t)\in\mathbb R^{2d}$ as in \eqref{eq:kou_intro} with $\sigma_\infty = 0.4$ and $\pi_v = \mathcal{N}(0,\Id_2)$. The last two panels display separately the position and velocity marginals
    of the same kinetic phase-space process.
    }
    \label{fig:gaussian_noising_comparison}
\end{figure}

Hence, the first source of error, controls the impact of replacing $\pi_T$ by $\pi_\infty$ in the reverse SDE and therefore writes for some distance $\dist$ on the space of probability measures as
$$
\mathcal{E}_{\rm init}
\eqdef
\dist \left( \pi_T Q_{0,T}, \pi_\infty Q_{0,T} \right) \eqsp.
%= \rmd \left( \pi_{\rm data}, \pi_\infty Q_{0,T} \right) \eqsp.
$$

% $$
% p_T(x)
% =
% \int_{\mathbb R^d}
% (2\pi\sigma_T^2)^{-d/2}
% \exp\!\left(
% -\frac{\|x-m_Tx_0\|^2}{2\sigma_T^2}
% \right)
% \pi_{\rm data}(\rmd x_0) \eqsp .
% $$

% The nature of this approximation depends on the noising process. In the
% variance-exploding case $\alpha=0$, one has
% $$
% \ora X_T=X_0+\sqrt{2\tau_T} Z \eqsp,
% \qquad
% \tau_T=\int_0^T \beta(t) \rmd t \eqsp.
% $$
% There is no invariant distribution in this case: the variance keeps increasing.
% The reference law \(\pi_\infty\) is therefore chosen as a large Gaussian, for
% example \(\mathcal N(0,2\tau_T I_d)\), which approximates \(\pi_T\) when the
% added noise dominates the scale of the data. In the variance-preserving case $\alpha>0$, the forward process is an
% Ornstein--Uhlenbeck-type diffusion. Such process is ergodic with invariant distribution
% $$
% \pi_\infty=\mathcal N(0,\alpha^{-1} \Id_d),
% $$
% and the terminal law satisfies
% $$
% \ora X_T
% =
% \rme^{-\alpha\tau_T}X_0
% +
% \sqrt{\frac{1- \rme^{-2\alpha\tau_T}}{\alpha}}\,Z.
% $$
% Hence the dependence on $X_0$ vanishes as $\tau_T\to\infty$, and
% $\pi_T$ converges weakly to $\pi_\infty$. In this case, replacing $\pi_T$ by
% \(\pi_\infty\) is genuinely a mixing approximation: the forward process forgets its initial condition. 

\paragraph{Approximation error.}

The ideal reverse kernel $Q_{s,t}$ is defined using the score $ \nabla \log p_t(x)$. In practice, this score depend on $\pi_{\rm data}$ and is therefore unknown. It is replaced by a learned parametric approximation $\score_\theta(x,t) : \R^d \times [0,T] \to \R^d$. For the forward noising processes considered here, the conditional transition
densities $p_{t|0}(\cdot\mid x_0)$ are Gaussian. Hence its conditional score $\nabla_x\log q_{t|0}(x\mid x_0)$ is known analytically; see \eqref{eq:gaussian_smoothing_score}.  Denoising score matching can therefore be applied at each time $t$. Training over a continuum of noise levels amounts to averaging the fixed-time conditional score matching objective over $t$. Namely, for a weighting function $\lambda : [0,T] \to \R_{>0}$, one considers: 
\begin{equation} \label{eq:loss_csm_chap0_sec3}
\mathcal L_{\rm CSM}(\theta) = \frac{1}{2} \int_0^T \lambda(t) \E \left[ \left\| \score_\theta( \ora X_t,t) - \nabla \log p_{t|0} (\ora X_t\mid \ora X_0) \right\|^2 \right] \rmd t \eqsp.
\end{equation}
In practice, \eqref{eq:loss_csm_chap0_sec3} is estimated using a Monte Carlo approximation.
% If $\lambda$ is normalized, one samples a random time $\tau$ with density $\lambda(t)$,
% then draws
% $$
% \ora X_0\sim\pi_{\rm data},
% \qquad
% \ora X_{\tau}\sim q_{\tau|0}(\cdot\mid \ora X_0),
% $$
% and performs a stochastic-gradient update using the corresponding integrand in
% \eqref{eq:loss_csm_chap0_sec3}. 
% -------------------- 
More explicitly, let $\nu$ be a probability distribution on $[0,T]$ with density $\rho$ such that $\rho(t)>0$ whenever $\lambda(t)>0$. Given $n$
Monte Carlo samples, draw independently
$$
\tau_i\sim \nu,\qquad
\ora X_0^i\sim \pi_{\rm data},
\qquad
\ora X_{\tau_i}^i\sim p_{\tau_i|0}(\cdot\mid \ora X_0^i),
\qquad i=1,\ldots,n \eqsp.
$$
Then, a Monte Carlo estimator of \eqref{eq:loss_csm_chap0_sec3} is given by
\begin{equation}\label{eq:empirical_csm_general}
\widehat{\mathcal L}_{{\rm CSM}}^n(\theta)
\eqdef
\frac{1}{2n}
\sum_{i=1}^n
\frac{\lambda(\tau_i)}{\rho(\tau_i)}
\left\|
\score_\theta(\ora X_{\tau_i}^i,\tau_i)- \nabla_x\log p_{\tau_i|0}
\big(\ora X_{\tau_i}^i\mid \ora X_0^i\big)
\right\|^2 \eqsp.
\end{equation}
In practice, this empirical objective is used to perform a stochastic-gradient update
$$
\theta_{m+1} = \theta_m - \eta_m \nabla_\theta \widehat{\mathcal L}_{{\rm CSM},n}(\theta_m) \eqsp,
$$
where $\eta_m>0$ is the learning rate at iteration $m$.
% In the common case where \(\lambda\) is normalized and one samples
% \(\tau_i\) directly with density \(\rho=\lambda\), this reduces to
% \begin{equation}\label{eq:empirical_csm_normalized}
% \widehat{\mathcal L}_{{\rm CSM},n}(\theta)
% =
% \frac{1}{2n}
% \sum_{i=1}^n
% \left\|
% \score_\theta(\ora X_{\tau_i}^i,\tau_i)
% -
% \nabla_x\log q_{\tau_i|0}
% \big(\ora X_{\tau_i}^i\mid \ora X_0^i\big)
% \right\|^2 .
% \end{equation}

Once the neural parametric approximation is learned, the approximate reverse SDE $(\ola X_t^\theta)_{t \in [0,T]}$ solves
\begin{equation}\label{eq:approx-reverse-sde}
\rmd \ola X_t^\theta
=
\left[ -b(T-t,\ola X_t^\theta)
+ a_{T-t} \score_\theta(\ola X_t^\theta,T-t) \right]\rmd t
+ G_{T-t} \rmd  B_t \eqsp,
\end{equation}
with $a_{T-t} =G_{T-t} G_{T-t}^\top$. The diffusion coefficient is the same in both \eqref{eq:sgm-backward-exact} and \eqref{eq:approx-reverse-sde}. Hence the difference between the two dynamics can be seen as a drift perturbation for $t \in [0,T]$,
$$
a_{T-t}
\left[ \score_\theta(x,T-t) - \nabla \log p_{T-t}(x) \right] \eqsp.
$$
Thus an error in the learned score is not merely a static regression error: once inserted into the reverse SDE, it becomes a dynamical error that is accumulated and propagated along the sampling trajectory. Therefore, denoting by
$$
Q_{s,t}^\theta(x,A)
\eqdef
\mathbb P\left(\ola X_t^{\theta,s,x}\in A\right),
\qquad 0\le s\le t\le T \eqsp,
$$
the transition kernels of \eqref{eq:approx-reverse-sde} for $x\in \R^d$ and $A \in \mathcal{B}(\R^d)$, for an initial distribution $\mu \in \mathcal{P} (\R^d)$, and for a distance $\dist$ on the space of probability measures, the score approximation error is
$$
\mathcal{E}_{\rm approx}
\eqdef
\dist \left(\mu Q_{0,T},\mu Q_{0,T}^\theta\right) \eqsp.
$$
% In the full sampling procedure, one usually takes $\mu=\pi_\infty$, because the
% initialization error has already replaced the exact terminal law $\pi_T$ by the
% tractable reference law $\pi_\infty$.

In fact, the term ``score approximation error'' encompasses several distinct sources of error. At the population level, denoising score matching identifies the true score only after
restricting the search to a model class $\mathcal{S}_\Theta= \{\score_\theta:\theta\in\Theta\}$. If the true score field $\nabla\log p_t$ does not belong to this class, one incurs an \emph{approximation error}. This includes the effect of the neural network architecture, such as its width, depth, and parametrization. With finitely
many training samples, the empirical DSM objective differs from its population counterpart, which gives a \emph{statistical} or \emph{generalization error}. Finally, the empirical DSM objective is optimized by a numerical procedure, typically a stochastic-gradient method. The
parameter returned by this procedure need not be a global minimizer of the
empirical objective. The gap between the achieved empirical loss and the best empirical loss within the model class is the \emph{optimization error}.

Recent works have begun to analyze this learning component more directly. Some results focus on score approximation, estimation, distribution recovery,
and minimax statistical guarantees
\citep{oko2023diffusion,chen2023score, dou2024optimal, zhang24bv, fu2025approximation,
stephanovitch2025generalization}, others establish generalization bounds for
diffusion or score-based models \citep{li2023generalization,yakovlev2025generalization,dupuisalgorithm},
while a complementary line studies optimization dynamics for neural-network
score estimation or score matching \citep{han2024neural,zhang2025convergence,strasman2026sgd}.

\paragraph{Discretization error.}

After replacing the true score by a learned approximation $s_\theta$, the approximate reverse SDE is given by \eqref{eq:approx-reverse-sde}. Even if $s_\theta$ were known exactly, this SDE usually cannot be simulated exactly. Indeed, its drift depends on a high-dimensional non-linear vector field and the associated transition kernels are not available in closed form. One therefore replaces the continuous-time dynamics by a discrete-time Markov chain approximation. The resulting discrepancy between the law of the continuous process and the law of the numerical scheme is the \emph{discretization error}.
More precisely, let
$$
0=t_0<t_1<\cdots<t_N=T \eqsp,
$$
be a time grid of $[0,T]$ with $N \in \mathbb{N}^*$ time steps. For $t\in[t_k,t_{k+1})$, we write
$$
\underline t \eqdef t_k \eqsp,
\qquad h_k\eqdef t_{k+1}- t_k \eqsp.
$$
The notation $\underline t$ denotes recursively the last grid point before $t \in [0,T]$ and $h_k$ is the stepsize. Different numerical schemes correspond to different ways of freezing or integrating the coefficients of \eqref{eq:approx-reverse-sde} on each interval
$[t_k,t_{k+1}]$ so that the one-step transition kernel between $t_k$ and $t_{k+1}$ becomes available. Let 
$$
(Z_k)_{k\ge 1}\stackrel{\mathrm{i.i.d.}}{\sim}\mathcal N(0_d,\Id_d) \eqsp,
$$
be a sequence of standard Gaussian random vectors of $\R^d$, independent of the initial condition. The \emph{Euler-Maruyama discretization} \citep{kloeden1992numerical, pages2018numerical} of \eqref{eq:approx-reverse-sde} writes as
\begin{multline} \label{eq:em_general_reverse}
\bar X_{t_{k+1}}^{\theta,\mathrm{EM}}
=
\bar X_{t_k}^{\theta,\mathrm{EM}}
+
h_k \left[ -b \left( T- t_k, \bar X_{t_k}^{\theta,\mathrm{EM}} \right) + a_{T-t_k} \score_\theta \left( \bar X_{t_k}^{\theta,\mathrm{EM}}, T - t_k \right)  \right]
 \\
+
\sqrt{h_k} G_{T-t_k} Z_{k+1} \eqsp.
\end{multline}
The discrete-time Markov chain \eqref{eq:em_general_reverse} coincides at grid points
$t_0,\ldots,t_N$ with the following continuous-time interpolation: for
$t\in[t_k,t_{k+1})$,
\begin{equation}\label{eq:em_interpolation_general_reverse}
\rmd \bar X_t^{\theta,\mathrm{EM}}
= \left[ -b(T-\underline t,\bar X_{\underline t}^{\theta,\mathrm{EM}}) + a_{T-\underline t}\score_\theta
\left( \bar X_{\underline t}^{\theta,\mathrm{EM}}, T-\underline t \right) \right]\rmd t + G_{T-\underline t} \rmd  B_t \eqsp.
\end{equation}
The recursion \eqref{eq:em_general_reverse} is the object used in practice: each step is explicit once the score network is evaluated and independent Gaussian vector are sampled. The interpolation \eqref{eq:em_interpolation_general_reverse}, on the other hand, is mostly used for theoretical analysis. It embeds the discrete Markov chain into continuous time so that both coincinde in distribution at grid points. This continuous viewpoint allows to compare the numerical scheme to the approximate reverse SDE using stochastic calculus tools.

\begin{figure}[t]
    \centering
    \includegraphics[width=0.82\textwidth]{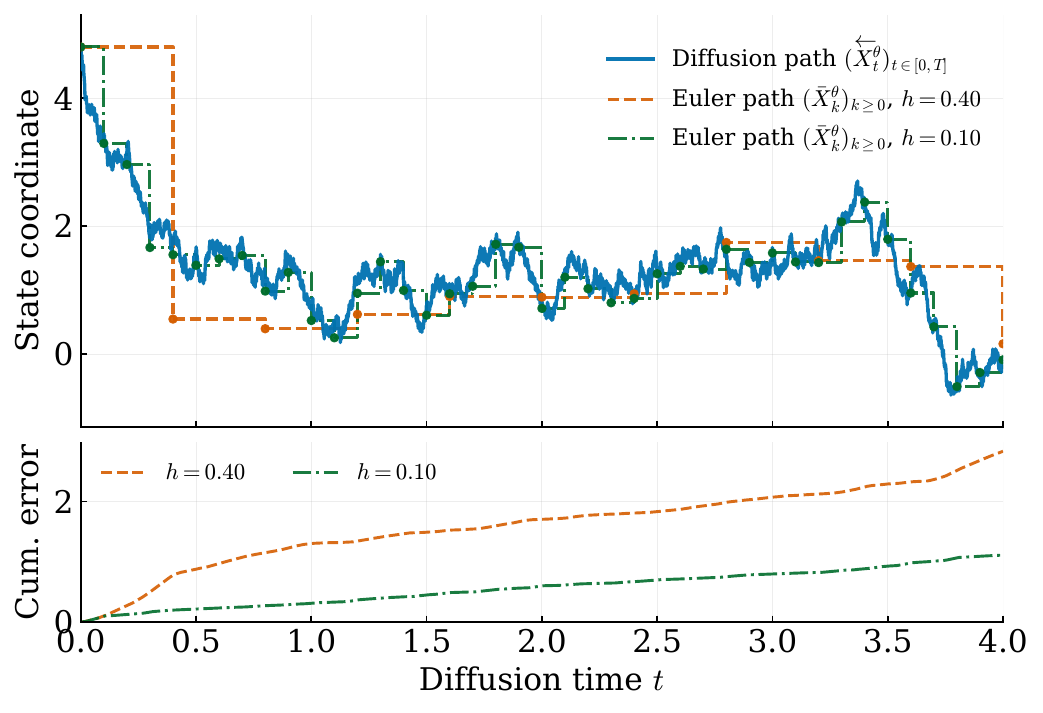}
    \caption{Illustration of time discretization. The blue curve represents a high-resolution reference trajectory of the approximate reverse dynamics, while the dashed curves are Euler--Maruyama stepwise approximations with different step sizes, driven by the same Brownian realization.
    The lower panel displays the approximated cumulative pathwise $L^1$ discrepancy between each stepwise interpolation and the reference trajectory.
    }
    \label{fig:reverse_sde_discretization}
\end{figure}

\begin{remark}[Structured integrators] \label{rem:structured_integrators}
Euler--Maruyama is the most generic explicit discretization: it freezes both the drift and diffusion coefficients on each interval $[t_k,t_{k+1}]$. It is however possible to design more structured integrators.
\begin{itemize}
    \item First, in the setting considered here, the diffusion matrix $G_{T-t}$ is deterministic and independent of the state. Hence the stochastic integral over one step is explicitly Gaussian and can, therefore, be integrated exactly:
$$
\int_{t_k}^{t_{k+1}} G_{T-s} \rmd  B_s
\sim
\mathcal N \left(0, \int_{t_k}^{t_{k+1}} G_{T-s}G_{T-s}^\top \rmd s \right) \eqsp.
$$
In the time-inhomogeneous scalar case, as in \eqref{eq:vp_ve_sde}, one can thefore consider the \emph{time-changed Euler scheme} \citep{strasman2026forgetting}:
For $t \in [t_k,t_{k+1}]$,
    \begin{equation*} %\label{eq:tc-interpolation}
    \rmd \bar X_t^{\theta,\mathrm{Tc}}
    =
    \bar\beta(t)
    \Big(
    \alpha \bar X_{\underline t}^{\theta,\mathrm{Tc}}
    +
    2s_\theta( \bar X_{\underline t}^{\theta,\mathrm{Tc}},T-\underline t)
    \Big)\rmd t
    +
    \sqrt{2\bar\beta(t)} \rmd  B_t \eqsp,
    \end{equation*}
    which is equivalent, at the grid points, to
    $$
    \bar X_{t_{k+1}}^{\theta,\mathrm{Tc}}
    =
    \bar X_{t_k}^{\theta,\mathrm{Tc}}
    +
    \Delta_k
    \left(
    \alpha \bar X_{t_k}^{\theta,\mathrm{Tc}}
    +
    2s_\theta(\bar X_{t_k}^{\theta,\mathrm{Tc}},T-t_k)
    \right)
    +
    \sqrt{2\Delta_k} Z_{k+1} \eqsp,
    $$
where $\Delta_k$ is the accumulated noise intensity between $t_k$ and $t_{k+1}$
$$
\Delta_k \eqdef \int_{t_k}^{t_{k+1}}\bar\beta(s) \rmd s \eqsp.
$$
\item Second, if the drift contains a linear part that can be integrated exactly, exponential integrators may be preferable. For instance, when
$$
\bar b_\theta(t,x)=A_t x+r_t+c_\theta(t,x) \eqsp,
$$
one can integrate the linear dynamics associated with $A_t x+r_t$ exactly and freeze only the nonlinear term $c_\theta$ on each step \citep{hochbruck2010exponential,strasman2025an}.

\item Finally, for phase-space or kinetic models, one can exploit the position--velocity structure of the dynamics. Instead of applying Euler--Maruyama to the full drift as a black box, one may split the dynamics into simpler substeps. Such \emph{splitting schemes} compose elementary flows, some of which can be integrated exactly or more accurately than the full dynamics \citep{strang1968construction}. Leveraging the physical interpretation of kinetic systems, one can also use
structure-preserving integrators inspired by \emph{symplectic methods}. These methods are designed to preserve the geometry of Hamiltonian flows, and can be
used for the Hamiltonian-like substeps of the kinetic dynamics \citep{leimkuhler2004simulating,tuckerman2023statistical}. This is particularly
natural for CLD-type score-based models, which evolve on a position--velocity phase space.
\end{itemize}
\end{remark}
For any chosen numerical scheme, we write its one-step update as a measurable map
\begin{equation} \label{eq:one_step_discr_update}
\Psi_k^\theta:\R^d\times\R^d\to\R^d \eqsp,
\qquad \bar X_{t_{k+1}}^\theta = \Psi_k^\theta(\bar X_{t_k}^\theta,Z_{k+1}) \eqsp,
\end{equation}
where $Z_{k+1}\sim\mathcal N(0,\Id_d)$ contains the independent random variables needed by the scheme. The one-step Markov kernel associated with $\Psi_k^\theta$ is defined, for every bounded measurable function $f$, by
$$
\bar Q_{k+1|k}^\theta f(x) \eqdef \mathbb E \left[ f \left( \Psi_k^\theta(x,Z_{k+1}) \right) \right] \eqsp.
$$
Equivalently, if
$$
\bar X_{t_k}^{\theta,k,x}=x \eqsp, \qquad
\bar X_{t_{\ell+1}}^{\theta,k,x} = \Psi_\ell^\theta
\left( \bar X_{t_\ell}^{\theta,k,x},Z_{\ell+1} \right) \eqsp,
\qquad \ell=k,\ldots,N-1,
$$
then $\bar Q_{k+1|k}^\theta f(x) = \mathbb E \left[f\left(\bar X_{t_{k+1}}^{\theta,k,x}\right) \right]$. For $0\le k\le \ell\le N$, we define the composed kernel by
$$
\bar Q_{k:\ell}^\theta f(x)
\eqdef
\int_{\R^{d(\ell-k)}}
f(x_\ell)
\prod_{r=k}^{\ell-1}
\bar Q_{r+1|r}^\theta(x_r,\rmd x_{r+1}) \eqsp,
\qquad x_k=x \in \R^d \eqsp.
$$
By convention, $\bar Q_{k:k}^\theta f(x)=f(x)$. Finally, for a distance $\dist$ on probability measures, the discretization error is
$$
\mathcal E_{\rm discr}
\eqdef \dist \left( \mu Q_{0,T}^\theta, \mu \bar Q_{0:N}^\theta \right) \eqsp ,
\qquad \mu\in\mathcal P(\R^d) \eqsp,
$$
where $Q_{0,T}^\theta$ is the transition kernel of the continuous approximate reverse SDE \eqref{eq:approx-reverse-sde}.

\paragraph{SGMs in practice.}

Once the score network has been trained, the practical sampler is obtained by initializing the discrete reverse chain from the reference distribution $\pi_\infty$ and by using the learned score $\score_\theta$ in the numerical
reverse dynamics. If $\bar Q_{0:N}^\theta$ denotes the Markov kernel of the $N$-step discrete sampler, then the generated distribution is
$$
\widehat\pi_N^{\theta,\infty}
=
\pi_\infty \bar Q_{0:N}^\theta \eqsp.
$$
It takes into account the three approximations discussed in the last paragraphs: initialization, approximation and discretization errors. The generic Score-based Generative Modeling pipeline is summarized in Algorithm~\ref{alg:generic_sgm}.

\begin{algorithm}[t]
\caption{Generic SGM pipeline}
\label{alg:generic_sgm}
\small

\noindent
\textbf{Model choice.} Choose a noising horizon $T>0$ and a forward process
$$
\rmd \ora X_t=b(t,\ora X_t)\,\rmd t+G_t\,\rmd B_t \eqsp,
\qquad \ora X_0\sim\pi_{\rm data} \eqsp.
$$
This choice determines the corruption kernels $q_{t|0}$ and the reverse-time dynamics.
Choose a tractable reference law $\pi_\infty$, a time law $\nu(\rmd t)=\rho(t)\rmd t$ on $[0,T]$,
weights $\lambda(t)$, a grid $0=t_0<\cdots<t_N=T$, and one-step maps
$\Psi_k^\theta$ discretizing the approximate reverse SDE.

\vspace{0.8em}

\begin{minipage}[t]{0.56\textwidth}
\textbf{Training}
\begin{algorithmic}[1]
\Statex \textbf{Input:} dataset $\mathcal D$; model $\score_\theta$; initial parameter $\theta_0$; minibatch size $B$; learning rates $(\eta_m)_{m=0}^{K-1}$; number of iterations $K$.
\Statex \textbf{Output:} trained score $\score_{\hat \theta}$.
\For{$m=0,\ldots,K-1$}
    \WrapState{Sample a minibatch $X_0^i\in\mathcal D$, $i=1,\ldots,B$, times $\tau_i\sim\nu$ and forward realisations $X_{\tau_i}^i\sim q_{\tau_i|0}(\cdot\mid X_0^i)$.}
    \WrapState{Compute targets
    $$
    r_i=\nabla_x\log q_{\tau_i|0}(X_{\tau_i}^i\mid X_0^i).
    $$}
    \State Form the minibatch loss
    $$
    \widehat{\mathcal L}_m(\theta_m)
    =
    \tfrac{1}{2B}
    \sum_{i=1}^B
    \frac{\lambda(\tau_i)}{\rho(\tau_i)}
    \left\|
    \score_{\theta_m}(X_{\tau_i}^i,\tau_i)-r_i
    \right\|^2
    $$
    \WrapState{Update
    \[
    \theta_{m+1}
    =
    \theta_m-\eta_m\nabla_\theta
    \widehat{\mathcal L}_m(\theta_m)
    \]}
\EndFor
\State Set $\hat\theta\gets\theta_K$.
\State \Return $\score_{\hat\theta}$.
\end{algorithmic}
\end{minipage}
\hfill
\begin{minipage}[t]{0.4\textwidth}
\textbf{Sampling}
\begin{algorithmic}[1]
\Statex \textbf{Input:} trained score $\score_{\hat\theta}$.
\Statex \textbf{Output:} $Y_N \sim \widehat\pi_N^\theta$.
\State Sample $Y_0\sim\pi_\infty$.
\For{$k=0,\ldots,N-1$}
    \State Sample $Z_{k+1}\sim\mathcal N(0,\Id_d)$.
    \State Update
    \[
    Y_{k+1}\gets \Psi_k^{\hat\theta}(Y_k,Z_{k+1}).
    \]
\EndFor
\State \Return $Y_N$.
\end{algorithmic}
\end{minipage}
\end{algorithm}

\subsection{Error analysis}

\paragraph{Error decomposition.}

As discussed above, studying the stability and convergence of SGM algorithms amounts to comparing the a practical SGM implementation whose distribution is given by $\widehat\pi_N^{\theta,\infty}$ with the ideal time-reversed process whose terminal marginal distribution is $\pi_{\rm data}$. In other words, for some distance $\dist$ on the space of probability measure, its amounts to controlling
\begin{align*}
\dist \left(\pi_{\rm data},\widehat \pi_N^{\theta,\infty}\right)
&=
\dist \left( \pi_T Q_{0,T}, \pi_\infty \bar Q_{0:N}^{\theta}  \right) \\
&\le
\underbrace{
\dist \left( \pi_T Q_{0,T} , \pi_\infty Q_{0,T}\right)
}_{\text{Initialization error}}
+
\underbrace{
\dist \left(\pi_\infty Q_{0,T},\pi_\infty Q^\theta_{0,T}\right)
}_{\text{Score approximation error}} +
%&\qquad+
\underbrace{
\dist \left( \pi_\infty Q^\theta_{0,T} , \pi_\infty \bar Q_{0:N}^{\theta}  \right)
}_{\text{Discretization error}} \eqsp .
\end{align*}

For divergences that are not metrics, such as the Kullback--Leibler divergence,
the displayed inequality should be understood as a schematic decomposition
rather than a literal triangle inequality. In fact, the strength and
interpretation of such a bound depend strongly on the choice of $\dist$. We therefore briefly review the most common theoretical distances used for SGM analysis, as well as empirical metrics used to evaluate generated samples in practice.

\paragraph{Metrics and divergences between probability measures.}

The interpretation of the previous error decomposition depends strongly on the
choice of discrepancy $\dist$. We briefly recall some common choices used in the analysis of Score-based Generative Models. Let $\mu$ and $\nu$ be two probability measures on $\R^d$.

\begin{itemize}[leftmargin=*]
    \item \textbf{TV distance.}
    The total variation distance is
    $$
    \left\| \mu-\nu \right\|_{\rm TV} \eqdef \sup_{A\in\mathcal B(\mathbb R^d)} |\mu(A)-\nu(A)| \eqsp.
    $$
    It is a strong notion of convergence, controlling the discrepancy between probabilities assigned to arbitrary measurable events. Total variation bounds for SGMs are obtained, for instance, in \citet{de2021diffusion,chen2023sampling}.

    \item \textbf{Relative entropy.} If $\mu$ is absolutely continuous with respect to $\nu$, the relative entropy or Kullback--Leibler divergence is
    $$
    {\rm KL}(\mu\|\nu) \eqdef \int_{\mathbb R^d} \log \left(\frac{\rmd\mu}{\rmd\nu}\right) \mu(\rmd x) \eqsp,
    $$
    and ${\rm KL}(\mu \|\nu)=+\infty$ otherwise. Unlike total variation or
    Wasserstein distance, KL is not symmetric and does not satisfy a triangle
    inequality, but it is particularly well adapted to Girsanov-type arguments. KL-based analyses for SGMs
    include \citet{chen2023improved,benton2024nearly,conforti_kl,strasman2025an}. Moreover, relative entropy bounds directly imply total variation bounds through Pinsker's inequality:
    $$
    \|\mu-\nu\|_{\rm TV} \le \sqrt{\frac12 {\rm KL}(\mu\|\nu)} \eqsp.
    $$

    \item \textbf{$2$-Wasserstein.}
    For probability measures $\mu,\nu$ with finite second moments, the squared $2$-Wasserstein distance is
    $$
    \mathcal W_2^2(\mu,\nu) = \inf_{\substack{X\sim \mu,\;Y\sim \nu}} \mathbb E\|X-Y\|^2 \eqsp.
    $$
    Wasserstein distances are well suited to coupling arguments and to control of geometric discrepancies between distributions. Wasserstein analyses of SGMs include
    \citet{bruno2023diffusion, gaoW2, strasman2025an,silveri2025beyond,strasman2025wasserstein, beyler2025convergence}.

    \item \textbf{Weighted TV distance.} Given a measurable Lyapunov function $V:\mathbb R^d\to[1,\infty)$ and $b >0$, the weigthed total variation distance is
    $$
    \rho_b (\mu , \nu) = \int_{\mathbb R^d} (1 + bV(x)) |\mu-\nu|(\rmd x) \eqsp.
    $$
    Such weighted norms are natural in Markov-chain theory, in particular in
    Harris-type ergodic theorems based on Lyapunov drift and minorization
    conditions \citep{HairerMattingly2008}. They have also been used in the
    analysis of SGMs, for example in \citet{strasman2026forgetting,strasman2026smc}.
\end{itemize}
Other discrepancies are possible, and the choice of metric can substantially change both the meaning and the strength of a convergence statement. General relations between classical probability metrics are surveyed in \citet{gibbs2002choosing}.

\paragraph{Empirical evaluation.}
In applications, the true data distribution is unknown and the generated law is
available only through samples. One therefore uses empirical discrepancies
between a finite dataset and generated samples. These quantities should be
understood as evaluation metrics, not as direct estimates of the theoretical
distances above. Let
$$
\pi_{\rm data}^{(n)}
\eqdef
\frac1n\sum_{i=1}^n \delta_{x_i}
\qquad \text{and} \qquad
\hat \pi_{\rm gen}^{(m)}
\eqdef
\frac1m \sum_{j=1}^m \delta_{\hat x_j}
$$
denote the empirical measures associated with the data samples
$(x_i)_{i=1}^n$ and the generated samples $(\hat x_j)_{j=1}^m$, respectively. Common empirical metrics include:

\begin{itemize}[leftmargin=*]

    \item \textbf{Sliced-Wasserstein distance.}
    This distance compares one-dimensional projections of the two distributions. For $u\in\mathbb S^{d-1}$, let
    $P_u(x)\eqdef u^\top x$. The squared sliced $2$-Wasserstein distance is
    $$
    \operatorname{SW}_2^2(\mu,\nu)
    \eqdef \int_{\mathbb S^{d-1}} \mathcal W_2^2 \left( (P_u)_\#\mu,(P_u)_\#\nu \right) \sigma(\rmd u) \eqsp,
    $$
    where $\sigma$ is the uniform measure on the sphere. In practice, the integral over directions is approximated by Monte Carlo: for $u_1,\ldots,u_L\stackrel{\mathrm{i.i.d.}}{\sim}\sigma$,
    \begin{equation} \label{eq:sliced_wasserstein_empirical}
    \operatorname{SW}_{2,L}^2 \left(\pi_{\rm data}^{(n)},\hat\pi_{\rm gen}^{(m)} \right)
    \eqdef \frac1L \sum_{\ell=1}^L \mathcal W_2^2 \left( (P_{u_\ell})_\#\pi_{\rm data}^{(n)}, (P_{u_\ell})_\#\hat\pi_{\rm gen}^{(m)} \right) \eqsp.
    \end{equation}
Since, in one dimension Wasserstein distances can be written in terms of quantile functions, each Wasserstein distance can be computed efficiently by sorting the projected samples. Sliced Wasserstein metrics are commonly used in generative modeling applications \citep{rabin2011wasserstein,bonneel2015sliced,kolouri2016sliced,flamary2021pot}.

    \item \textbf{Max-sliced Wasserstein distance.}
    A related discrepancy replaces the average over directions in sliced Wasserstein by the most discriminative projection:
    $$
    \operatorname{MSW}_2^2(\mu,\nu)
    \eqdef     \sup_{u \in  \mathbb S^{d-1}} \mathcal W_2^2 \left( (P_u)_\#\mu,(P_u)_\#\nu \right) \eqsp. 
    $$
    Instead of averaging the discrepancy over random directions, one searches for the projection along which the two measures differ the most. In practice, a Monte Carlo approximation similar to \eqref{eq:sliced_wasserstein_empirical} can be implemented and the supremum over the sphere is approximated numerically, by gradient-based optimization over $u\in\mathbb S^{d-1}$ \citep{deshpande2019max,kolouri2019generalized}.

    \item \textbf{FID.}
    In image generation, a widely used evaluation metric is the Fréchet Inception
    Distance (FID) \citep{heusel2017gans}. Let $\phi$ be a fixed feature extractor,
    typically given by an intermediate layer of a pretrained Inception network.
    Given data samples $(x_i)_{i=1}^n$ and generated samples $(\hat x_j)_{j=1}^m$,
    one computes the feature vectors $\phi(x_i)$ and $\phi(\hat x_j)$ and fits
    Gaussian approximations to the two empirical feature distributions:
    $$
    \mathcal N(m_{\rm data}^{(n)},C_{\rm data}^{(n)}) \eqsp,
    \qquad
    \mathcal N(m_{\rm gen}^{(m)},C_{\rm gen}^{(m)}) \eqsp,
    $$
    where $m^{(n)}$ denotes the empirical mean and $C^{(n)}$ the empirical covariance.
    The FID is then defined as the squared $2$-Wasserstein distance between these
    two Gaussian distributions:
    \begin{multline*}
    \operatorname{FID}
    \eqdef
    \|m_{\rm data}^{(n)}-m_{\rm gen}^{(m)}\|^2 
    + \operatorname{Tr} \left( C_{\rm data}^{(n)} + C_{\rm gen}^{(m)} -
    2\left( (C_{\rm data}^{(n)})^{1/2} C_{\rm gen}^{(m)} (C_{\rm data}^{(n)})^{1/2} \right)^{1/2}
    \right) \eqsp.
    \end{multline*}
    Thus, FID does not compare the original data and generated distributions directly. It compares Gaussian approximations of their distributions after a
    learned feature embedding. It should be interpreted as a practical image-quality metric.
\end{itemize}
Many other empirical evaluation metrics have been proposed, and the design of reliable metrics for generative models remains an active area of research.

\newpage

\bibliographystyle{plainnat}
\bibliography{biblio}

\begin{thebibliography}{108}
\providecommand{\natexlab}[1]{#1}
\providecommand{\url}[1]{\texttt{#1}}
\expandafter\ifx\csname urlstyle\endcsname\relax
  \providecommand{\doi}[1]{doi: #1}\else
  \providecommand{\doi}{doi: \begingroup \urlstyle{rm}\Url}\fi

\bibitem[Ackley et~al.(1985)Ackley, Hinton, and Sejnowski]{ackley1985learning}
David~H Ackley, Geoffrey~E Hinton, and Terrence~J Sejnowski.
\newblock A learning algorithm for boltzmann machines.
\newblock \emph{Cognitive science}, 9\penalty0 (1):\penalty0 147--169, 1985.

\bibitem[Albergo and Vanden-Eijnden(2023)]{albergo2023building}
Michael Albergo and Eric Vanden-Eijnden.
\newblock Building normalizing flows with stochastic interpolants.
\newblock In \emph{ICLR 2023 Conference}, 2023.

\bibitem[Anderson(1982)]{anderson1982}
B.~D. Anderson.
\newblock Reverse-time diffusion equation models.
\newblock \emph{Stochastic Processes and their Applications}, 12\penalty0 (3):\penalty0 313--326, 1982.

\bibitem[Arnold et~al.(1989)Arnold, Vogtmann, and Weinstein]{arnold1989mathematical}
Vladimir~Igorevich Arnold, Karen Vogtmann, and Alan Weinstein.
\newblock \emph{Mathematical methods of classical mechanics}, volume~60.
\newblock Springer, 1989.

\bibitem[Bengio et~al.(2003)Bengio, Ducharme, Vincent, and Jauvin]{bengio2003neural}
Yoshua Bengio, R{\'e}jean Ducharme, Pascal Vincent, and Christian Jauvin.
\newblock A neural probabilistic language model.
\newblock \emph{Journal of machine learning research}, 3\penalty0 (Feb):\penalty0 1137--1155, 2003.

\bibitem[Benton et~al.(2024)Benton, Bortoli, Doucet, and Deligiannidis]{benton2024nearly}
Joe Benton, Valentin~De Bortoli, Arnaud Doucet, and George Deligiannidis.
\newblock Nearly \$d\$-linear convergence bounds for diffusion models via stochastic localization.
\newblock In \emph{The Twelfth International Conference on Learning Representations}, 2024.

\bibitem[Beyler and Bach(2025)]{beyler2025convergence}
Eliot Beyler and Francis Bach.
\newblock Convergence of deterministic and stochastic diffusion-model samplers: A simple analysis in wasserstein distance.
\newblock \emph{arXiv preprint arXiv:2508.03210}, 2025.

\bibitem[Bonneel et~al.(2015)Bonneel, Rabin, Peyr{\'e}, and Pfister]{bonneel2015sliced}
Nicolas Bonneel, Julien Rabin, Gabriel Peyr{\'e}, and Hanspeter Pfister.
\newblock Sliced and radon wasserstein barycenters of measures.
\newblock \emph{Journal of Mathematical Imaging and Vision}, 51\penalty0 (1):\penalty0 22--45, 2015.

\bibitem[Bottou et~al.(2018)Bottou, Curtis, and Nocedal]{bottou2018optimization}
L{\'e}on Bottou, Frank~E Curtis, and Jorge Nocedal.
\newblock Optimization methods for large-scale machine learning.
\newblock \emph{SIAM review}, 60\penalty0 (2):\penalty0 223--311, 2018.

\bibitem[Bovier et~al.(2004)Bovier, Eckhoff, Gayrard, and Klein]{bovier2004metastability}
Anton Bovier, Michael Eckhoff, V{\'e}ronique Gayrard, and Markus Klein.
\newblock Metastability in reversible diffusion processes i: Sharp asymptotics for capacities and exit times.
\newblock \emph{Journal of the European Mathematical Society}, 6\penalty0 (4):\penalty0 399--424, 2004.

\bibitem[Box and Muller(1958)]{box1958generation}
George E.~P. Box and Mervin~E. Muller.
\newblock A note on the generation of random normal deviates.
\newblock \emph{The Annals of Mathematical Statistics}, 29\penalty0 (2):\penalty0 610--611, 1958.

\bibitem[Bruno et~al.(2025)Bruno, Zhang, Lim, Akyildiz, and Sabanis]{bruno2023diffusion}
Stefano Bruno, Ying Zhang, Dong-Young Lim, {\"O}mer~Deniz Akyildiz, and Sotirios Sabanis.
\newblock On diffusion-based generative models and their error bounds: The log-concave case with full convergence estimates.
\newblock \emph{Transactions on Machine Learning Research}, 2025.

\bibitem[Cattiaux et~al.(2023)Cattiaux, Conforti, Gentil, and L{\'e}onard]{cattiaux}
Patrick Cattiaux, Giovanni Conforti, Ivan Gentil, and Christian L{\'e}onard.
\newblock {Time reversal of diffusion processes under a finite entropy condition}.
\newblock \emph{{Annales de l'Institut Henri Poincar{\'e} (B) Probabilit{\'e}s et Statistiques}}, 59\penalty0 (4):\penalty0 1844--1881, 2023.

\bibitem[Chen et~al.(2023{\natexlab{a}})Chen, Lee, and Lu]{chen2023improved}
Hongrui Chen, Holden Lee, and Jianfeng Lu.
\newblock Improved analysis of score-based generative modeling: User-friendly bounds under minimal smoothness assumptions.
\newblock In \emph{International Conference on Machine Learning}, pages 4735--4763. PMLR, 2023{\natexlab{a}}.

\bibitem[Chen et~al.(2023{\natexlab{b}})Chen, Huang, Zhao, and Wang]{chen2023score}
Minshuo Chen, Kaixuan Huang, Tuo Zhao, and Mengdi Wang.
\newblock Score approximation, estimation and distribution recovery of diffusion models on low-dimensional data.
\newblock In \emph{International Conference on Machine Learning}, pages 4672--4712. PMLR, 2023{\natexlab{b}}.

\bibitem[Chen et~al.(2023{\natexlab{c}})Chen, Chewi, Li, Li, Salim, and Zhang]{chen2023sampling}
Sitan Chen, Sinho Chewi, Jerry Li, Yuanzhi Li, Adil Salim, and Anru Zhang.
\newblock Sampling is as easy as learning the score: theory for diffusion models with minimal data assumptions.
\newblock In \emph{The Eleventh International Conference on Learning Representations}, 2023{\natexlab{c}}.

\bibitem[Cheng and Bartlett(2018)]{ChengBartlett2018}
Xiang Cheng and Peter Bartlett.
\newblock Convergence of langevin mcmc in kl-divergence.
\newblock In \emph{Proceedings of Algorithmic Learning Theory}, volume~83 of \emph{Proceedings of Machine Learning Research}, pages 186--211, 2018.

\bibitem[Chicone(2006)]{chicone2006ordinary}
Carmen Chicone.
\newblock \emph{Ordinary Differential Equations with Applications}, volume~34 of \emph{Texts in Applied Mathematics}.
\newblock Springer, New York, 2 edition, 2006.

\bibitem[Chung et~al.(1969)Chung, Walsh, et~al.]{chung1969reverse}
Kai~Lai Chung, John~B Walsh, et~al.
\newblock To reverse a markov process.
\newblock \emph{Acta Math}, 123\penalty0 (1):\penalty0 225--251, 1969.

\bibitem[Conforti et~al.(2025)Conforti, Durmus, and Gentiloni-Silveri]{conforti_kl}
Giovanni Conforti, Alain Durmus, and Marta Gentiloni-Silveri.
\newblock Kl convergence guarantees for score diffusion models under minimal data assumptions.
\newblock \emph{SIAM Journal on Mathematics of Data Science}, 7\penalty0 (1):\penalty0 86--109, 2025.

\bibitem[Dalalyan(2017)]{dalalyan2017further}
Arnak~S. Dalalyan.
\newblock Further and stronger analogy between sampling and optimization: {L}angevin {M}onte {C}arlo and gradient descent, 2017.

\bibitem[Dalalyan and Tsybakov(2012)]{dalalyan2012sparse}
Arnak~S Dalalyan and Alexandre~B Tsybakov.
\newblock Sparse regression learning by aggregation and langevin monte-carlo.
\newblock \emph{Journal of Computer and System Sciences}, 78\penalty0 (5):\penalty0 1423--1443, 2012.

\bibitem[De~Bortoli et~al.(2021)De~Bortoli, Thornton, Heng, and Doucet]{de2021diffusion}
Valentin De~Bortoli, James Thornton, Jeremy Heng, and Arnaud Doucet.
\newblock Diffusion schr{\"o}dinger bridge with applications to score-based generative modeling.
\newblock \emph{Advances in neural information processing systems}, 34:\penalty0 17695--17709, 2021.

\bibitem[Deshpande et~al.(2019)Deshpande, Hu, Sun, Pyrros, Siddiqui, Koyejo, Zhao, Forsyth, and Schwing]{deshpande2019max}
Ishan Deshpande, Yuan-Ting Hu, Ruoyu Sun, Ayis Pyrros, Nasir Siddiqui, Sanmi Koyejo, Zhizhen Zhao, David Forsyth, and Alexander~G Schwing.
\newblock Max-sliced wasserstein distance and its use for gans.
\newblock In \emph{Proceedings of the IEEE/CVF conference on computer vision and pattern recognition}, pages 10648--10656, 2019.

\bibitem[Dhariwal and Nichol(2021)]{beatsGAN}
Prafulla Dhariwal and Alexander~Quinn Nichol.
\newblock Diffusion models beat {GAN}s on image synthesis.
\newblock In A.~Beygelzimer, Y.~Dauphin, P.~Liang, and J.~Wortman Vaughan, editors, \emph{Advances in Neural Information Processing Systems}, 2021.

\bibitem[Dinh et~al.(2017)Dinh, Sohl-Dickstein, and Bengio]{dinh2017density}
Laurent Dinh, Jascha Sohl-Dickstein, and Samy Bengio.
\newblock Density estimation using real nvp.
\newblock In \emph{International Conference on Learning Representations}, 2017.

\bibitem[Dockhorn et~al.(2022)Dockhorn, Vahdat, and Kreis]{dockhorn2022score}
Tim Dockhorn, Arash Vahdat, and Karsten Kreis.
\newblock Score-based generative modeling with critically-damped langevin diffusion.
\newblock In \emph{International Conference on Learning Representations}, 2022.

\bibitem[Doob(1942)]{doob1942brownian}
J.~L. Doob.
\newblock The brownian movement and stochastic equations.
\newblock \emph{Annals of Mathematics (2)}, 43\penalty0 (2):\penalty0 351--369, April 1942.

\bibitem[Doob(1953)]{doob1953stochastic}
Joseph~L. Doob.
\newblock \emph{Stochastic Processes}.
\newblock John Wiley \& Sons, New York, 1953.

\bibitem[Dou et~al.(2024)Dou, Kotekal, Xu, and Zhou]{dou2024optimal}
Zehao Dou, Subhodh Kotekal, Zhehao Xu, and Harrison~H Zhou.
\newblock From optimal score matching to optimal sampling.
\newblock \emph{arXiv preprint arXiv:2409.07032}, 2024.

\bibitem[Dupuis et~al.(2025)Dupuis, Shariatian, Haddouche, Durmus, and Simsekli]{dupuisalgorithm}
Benjamin Dupuis, Dario Shariatian, Maxime Haddouche, Alain~Oliviero Durmus, and Umut Simsekli.
\newblock Algorithm-and data-dependent generalization bounds for diffusion models.
\newblock In \emph{Advances in Neural Information Processing Systems}, 2025.

\bibitem[Durmus and Moulines(2017)]{durmus2017nonasymptotic}
Alain Durmus and {\'E}ric Moulines.
\newblock Nonasymptotic convergence analysis for the unadjusted langevin algorithm.
\newblock \emph{The Annals of Applied Probability}, 27\penalty0 (3):\penalty0 1551--1587, 2017.

\bibitem[Durmus and Moulines(2019)]{durmus2019high}
Alain Durmus and {\'E}ric Moulines.
\newblock High-dimensional bayesian inference via the unadjusted langevin algorithm.
\newblock \emph{Bernoulli}, 25\penalty0 (4A), 2019.

\bibitem[Dynkin(1961)]{dynkin1961markov}
Eugene~B. Dynkin.
\newblock \emph{Theory of {Markov} Processes}.
\newblock Prentice-Hall, 1961.

\bibitem[Ethier and Kurtz(1986)]{ethier1986markov}
Stewart~N. Ethier and Thomas~G. Kurtz.
\newblock \emph{Markov Processes: Characterization and Convergence}.
\newblock Wiley, New York, 1986.

\bibitem[Flamary et~al.(2021)Flamary, Courty, Gramfort, Alaya, Boisbunon, Chambon, Chapel, Corenflos, Fatras, Fournier, et~al.]{flamary2021pot}
R{\'e}mi Flamary, Nicolas Courty, Alexandre Gramfort, Mokhtar~Z Alaya, Aur{\'e}lie Boisbunon, Stanislas Chambon, Laetitia Chapel, Adrien Corenflos, Kilian Fatras, Nemo Fournier, et~al.
\newblock Pot: Python optimal transport.
\newblock \emph{Journal of Machine Learning Research}, 22\penalty0 (78):\penalty0 1--8, 2021.

\bibitem[F{\"o}llmer and Wakolbinger(1986)]{follmer1986time}
H~F{\"o}llmer and A~Wakolbinger.
\newblock Time reversal of infinite-dimensional diffusions.
\newblock \emph{Stochastic processes and their applications}, 22\penalty0 (1):\penalty0 59--77, 1986.

\bibitem[F{\"o}llmer(1985)]{follmer1985entropy}
Hans F{\"o}llmer.
\newblock An entropy approach to the time reversal of diffusion processes.
\newblock In Michel M{\'e}tivier and {\'E}tienne Pardoux, editors, \emph{Stochastic Differential Systems}, volume~69 of \emph{Lecture Notes in Control and Information Sciences}, pages 156--163. Springer, 1985.

\bibitem[Fu and Lee(2025)]{fu2025approximation}
Guoji Fu and Wee~Sun Lee.
\newblock Approximation and generalization abilities of score-based neural network generative models for sub-gaussian distributions.
\newblock In \emph{Advances in Neural Information Processing Systems}, 2025.

\bibitem[Gao et~al.(2025)Gao, Nguyen, and Zhu]{gaoW2}
Xuefeng Gao, Hoang~M Nguyen, and Lingjiong Zhu.
\newblock Wasserstein convergence guarantees for a general class of score-based generative models.
\newblock \emph{Journal of machine learning research}, 26\penalty0 (43):\penalty0 1--54, 2025.

\bibitem[Gentiloni-Silveri and Ocello(2025)]{silveri2025beyond}
Marta Gentiloni-Silveri and Antonio Ocello.
\newblock Beyond log-concavity and score regularity: Improved convergence bounds for score-based generative models in w2-distance.
\newblock In \emph{Forty-second International Conference on Machine Learning}, 2025.

\bibitem[Gibbs and Su(2002)]{gibbs2002choosing}
Alison~L Gibbs and Francis~Edward Su.
\newblock On choosing and bounding probability metrics.
\newblock \emph{International statistical review}, 70\penalty0 (3):\penalty0 419--435, 2002.

\bibitem[Goodfellow et~al.(2014)Goodfellow, Pouget-Abadie, Mirza, Xu, Warde-Farley, Ozair, Courville, and Bengio]{goodfellow2014generative}
Ian~J Goodfellow, Jean Pouget-Abadie, Mehdi Mirza, Bing Xu, David Warde-Farley, Sherjil Ozair, Aaron Courville, and Yoshua Bengio.
\newblock Generative adversarial nets.
\newblock \emph{Advances in neural information processing systems}, 27, 2014.

\bibitem[Hairer and Mattingly(2008)]{HairerMattingly2008}
Martin Hairer and Jonathan~C. Mattingly.
\newblock Yet another look at harris' ergodic theorem for markov chains.
\newblock \emph{arXiv preprint arXiv:0810.2777}, 2008.

\bibitem[Han et~al.(2024)Han, Razaviyayn, and Xu]{han2024neural}
Yinbin Han, Meisam Razaviyayn, and Renyuan Xu.
\newblock Neural network-based score estimation in diffusion models: Optimization and generalization.
\newblock In \emph{International Conference on Learning Representations}, pages 42520--42558, 2024.

\bibitem[Haussmann and Pardoux(1986)]{haussmann1986time}
U.~G. Haussmann and E.~Pardoux.
\newblock Time reversal of diffusions.
\newblock \emph{The Annals of Probability}, 14\penalty0 (4):\penalty0 1188--1205, 1986.

\bibitem[Heusel et~al.(2017)Heusel, Ramsauer, Unterthiner, Nessler, and Hochreiter]{heusel2017gans}
Martin Heusel, Hubert Ramsauer, Thomas Unterthiner, Bernhard Nessler, and Sepp Hochreiter.
\newblock Gans trained by a two time-scale update rule converge to a local nash equilibrium.
\newblock \emph{Advances in neural information processing systems}, 30, 2017.

\bibitem[Hinton(2002)]{hinton2002training}
Geoffrey~E Hinton.
\newblock Training products of experts by minimizing contrastive divergence.
\newblock \emph{Neural computation}, 14\penalty0 (8):\penalty0 1771--1800, 2002.

\bibitem[Ho et~al.(2020)Ho, Jain, and Abbeel]{ho2020denoising}
Jonathan Ho, Ajay Jain, and Pieter Abbeel.
\newblock Denoising diffusion probabilistic models.
\newblock In \emph{Advances in Neural Information Processing Systems (NeurIPS)}, pages 6840--6851, 2020.

\bibitem[Hochbruck and Ostermann(2010)]{hochbruck2010exponential}
Marlis Hochbruck and Alexander Ostermann.
\newblock Exponential integrators.
\newblock \emph{Acta Numerica}, 19:\penalty0 209--286, 2010.

\bibitem[H{\"o}rmander(1967)]{hormander1967hypoelliptic}
Lars H{\"o}rmander.
\newblock Hypoelliptic second order differential equations.
\newblock \emph{Acta Mathematica}, 119:\penalty0 147--171, 1967.

\bibitem[Hyv{\"a}rinen and Dayan(2005)]{hyvarinen2005estimation}
Aapo Hyv{\"a}rinen and Peter Dayan.
\newblock Estimation of non-normalized statistical models by score matching.
\newblock \emph{Journal of Machine Learning Research}, 6\penalty0 (4), 2005.

\bibitem[Jolicoeur-Martineau et~al.(2021)Jolicoeur-Martineau, Pich{\'e}-Taillefer, Mitliagkas, and des Combes]{jolicoeur-martineau2021adversarial}
Alexia Jolicoeur-Martineau, R{\'e}mi Pich{\'e}-Taillefer, Ioannis Mitliagkas, and Remi~Tachet des Combes.
\newblock Adversarial score matching and improved sampling for image generation.
\newblock In \emph{International Conference on Learning Representations}, 2021.

\bibitem[Karras et~al.(2022)Karras, Aittala, Aila, and Laine]{edmkarras}
Tero Karras, Miika Aittala, Timo Aila, and Samuli Laine.
\newblock Elucidating the design space of diffusion-based generative models.
\newblock In S.~Koyejo, S.~Mohamed, A.~Agarwal, D.~Belgrave, K.~Cho, and A.~Oh, editors, \emph{Advances in Neural Information Processing Systems}, volume~35, pages 26565--26577. Curran Associates, Inc., 2022.

\bibitem[Kingma and Welling(2014)]{kingma2014autoencoding}
Diederik~P. Kingma and Max Welling.
\newblock Auto-encoding variational bayes.
\newblock In \emph{International Conference on Learning Representations}, 2014.

\bibitem[Kingma and Dhariwal(2018)]{kingma2018glow}
Durk~P Kingma and Prafulla Dhariwal.
\newblock Glow: Generative flow with invertible 1x1 convolutions.
\newblock \emph{Advances in neural information processing systems}, 31, 2018.

\bibitem[Kloeden and Platen(1992)]{kloeden1992numerical}
Peter~E. Kloeden and Eckhard Platen.
\newblock \emph{Numerical Solution of Stochastic Differential Equations}.
\newblock Springer, 1992.

\bibitem[Kolouri et~al.(2016)Kolouri, Zou, and Rohde]{kolouri2016sliced}
Soheil Kolouri, Yang Zou, and Gustavo~K Rohde.
\newblock Sliced wasserstein kernels for probability distributions.
\newblock In \emph{Proceedings of the IEEE Conference on Computer Vision and Pattern Recognition}, pages 5258--5267, 2016.

\bibitem[Kolouri et~al.(2019)Kolouri, Nadjahi, Simsekli, Badeau, and Rohde]{kolouri2019generalized}
Soheil Kolouri, Kimia Nadjahi, Umut Simsekli, Roland Badeau, and Gustavo Rohde.
\newblock Generalized sliced wasserstein distances.
\newblock \emph{Advances in neural information processing systems}, 32, 2019.

\bibitem[Langevin(1908)]{langevin1908brownian}
Paul Langevin.
\newblock Sur la th{\'e}orie du mouvement brownien.
\newblock \emph{Comptes Rendus de l'Acad{\'e}mie des Sciences}, 146:\penalty0 530--533, 1908.

\bibitem[Leimkuhler and Reich(2004)]{leimkuhler2004simulating}
Benedict Leimkuhler and Sebastian Reich.
\newblock \emph{Simulating hamiltonian dynamics}.
\newblock Cambridge university press, 2004.

\bibitem[L{\'e}onard et~al.(2014)L{\'e}onard, R{\oe}lly, and Zambrini]{leonard2014reciprocal}
Christian L{\'e}onard, Sylvie R{\oe}lly, and Jean-Claude Zambrini.
\newblock Reciprocal processes. a measure-theoretical point of view.
\newblock \emph{Probability Surveys}, 11:\penalty0 237--269, 2014.

\bibitem[Li et~al.(2023)Li, Li, Zhang, and Bian]{li2023generalization}
Puheng Li, Zhong Li, Huishuai Zhang, and Jiang Bian.
\newblock On the generalization properties of diffusion models.
\newblock In \emph{Advances in Neural Information Processing Systems}, volume~36, pages 2097--2127, 2023.

\bibitem[Lipman et~al.(2023)Lipman, Chen, Ben-Hamu, Nickel, and Le]{lipman2023flowmatching}
Yaron Lipman, Ricky T.~Q. Chen, Heli Ben-Hamu, Maximilian Nickel, and Matthew Le.
\newblock Flow matching for generative modeling.
\newblock In \emph{The Eleventh International Conference on Learning Representations}, 2023.

\bibitem[Liu et~al.(2023)Liu, Gong, and Liu]{liu2023rectifiedflow}
Xingchao Liu, Chengyue Gong, and Qiang Liu.
\newblock Flow straight and fast: Learning to generate and transfer data with rectified flow.
\newblock In \emph{The Eleventh International Conference on Learning Representations}, 2023.

\bibitem[Marsaglia and Tsang(2000)]{marsaglia2000ziggurat}
George Marsaglia and Wai~Wan Tsang.
\newblock The ziggurat method for generating random variables.
\newblock \emph{Journal of Statistical Software}, 5\penalty0 (8):\penalty0 1--7, 2000.

\bibitem[Martens et~al.(2012)Martens, Sutskever, and Swersky]{martens2012estimating}
James Martens, Ilya Sutskever, and Kevin Swersky.
\newblock Estimating the hessian by back-propagating curvature.
\newblock In \emph{Proceedings of the 29th International Coference on International Conference on Machine Learning}, pages 963--970, 2012.

\bibitem[Metafune et~al.(2002)Metafune, Pallara, and Priola]{metafune2002spectrum}
Giorgio Metafune, Diego Pallara, and Enrico Priola.
\newblock Spectrum of ornstein--uhlenbeck operators in \(l^p\) spaces with respect to invariant measures.
\newblock \emph{Journal of Functional Analysis}, 196\penalty0 (1):\penalty0 40--60, 2002.

\bibitem[Meyer(1967)]{meyer1967processus}
Paul-Andr{\'e} Meyer.
\newblock \emph{Processus de {Markov}}, volume~26 of \emph{Lecture Notes in Mathematics}.
\newblock Springer-Verlag, 1967.
\newblock ISBN 978-3-540-03901-3.

\bibitem[Millet et~al.(1989{\natexlab{a}})Millet, Nualart, and Sanz]{millet1989integration}
Annie Millet, David Nualart, and Marta Sanz.
\newblock Integration by parts and time reversal for diffusion processes.
\newblock \emph{The Annals of Probability}, 17\penalty0 (1):\penalty0 208--238, 1989{\natexlab{a}}.

\bibitem[Millet et~al.(1989{\natexlab{b}})Millet, Nualart, and Sanz]{millet1989time}
Annie Millet, David Nualart, and Marta Sanz.
\newblock Time reversal for infinite-dimensional diffusions.
\newblock \emph{Probability theory and related fields}, 82\penalty0 (3):\penalty0 315--347, 1989{\natexlab{b}}.

\bibitem[Miyasawa(1961)]{miyasawa1961empirical}
K.~Miyasawa.
\newblock An empirical bayes estimator of the mean of a normal population.
\newblock \emph{Bulletin of the International Statistical Institute}, 38:\penalty0 181--188, 1961.

\bibitem[Mou et~al.(2022)Mou, Flammarion, Wainwright, and Bartlett]{mou2022improved}
Wenlong Mou, Nicolas Flammarion, Martin~J Wainwright, and Peter~L Bartlett.
\newblock Improved bounds for discretization of langevin diffusions: Near-optimal rates without convexity.
\newblock \emph{Bernoulli}, 28\penalty0 (3):\penalty0 1577--1601, 2022.

\bibitem[Nagasawa(1964)]{nagasawa1964time}
Masao Nagasawa.
\newblock Time reversions of markov processes.
\newblock \emph{Nagoya Mathematical Journal}, 24:\penalty0 177--204, 1964.

\bibitem[Neal(2001)]{neal2001annealed}
Radford~M. Neal.
\newblock Annealed importance sampling.
\newblock \emph{Statistics and Computing}, 11\penalty0 (2):\penalty0 125--139, 2001.

\bibitem[Nelson(1967)]{nelson1967dynamical}
Edward Nelson.
\newblock \emph{Dynamical Theories of Brownian Motion}.
\newblock Princeton University Press, Princeton, NJ, 1967.

\bibitem[Oko et~al.(2023)Oko, Akiyama, and Suzuki]{oko2023diffusion}
Kazusato Oko, Shunta Akiyama, and Taiji Suzuki.
\newblock Diffusion models are minimax optimal distribution estimators.
\newblock In \emph{Proceedings of the 40th International Conference on Machine Learning}, volume 202 of \emph{Proceedings of Machine Learning Research}, pages 26517--26582. PMLR, 2023.

\bibitem[Ottobre et~al.(2015)Ottobre, Pavliotis, and Pravda-Starov]{ottobre2015degenerate}
Michela Ottobre, Grigorios~A. Pavliotis, and Karel Pravda-Starov.
\newblock Some remarks on degenerate hypoelliptic ornstein--uhlenbeck operators.
\newblock \emph{Journal of Mathematical Analysis and Applications}, 429\penalty0 (2):\penalty0 676--712, 2015.

\bibitem[Pag{\`e}s(2018)]{pages2018numerical}
Gilles Pag{\`e}s.
\newblock \emph{Numerical Probability: An Introduction with Applications to Finance}.
\newblock Universitext. Springer, Cham, 2018.
\newblock ISBN 978-3-319-90276-0.

\bibitem[Parisi(1981)]{parisi1981correlation}
Giorgio Parisi.
\newblock Correlation functions and computer simulations.
\newblock \emph{Nuclear Physics B}, 180\penalty0 (3):\penalty0 378--384, 1981.

\bibitem[Pierret and Galerne(2025)]{pierret2025diffusion}
Emile Pierret and Bruno Galerne.
\newblock Diffusion models for gaussian distributions: Exact solutions and wasserstein errors.
\newblock In \emph{International Conference on Machine Learning}, pages 49355--49381. PMLR, 2025.

\bibitem[Rabin et~al.(2011)Rabin, Peyr{\'e}, Delon, and Bernot]{rabin2011wasserstein}
Julien Rabin, Gabriel Peyr{\'e}, Julie Delon, and Marc Bernot.
\newblock Wasserstein barycenter and its application to texture mixing.
\newblock In \emph{Scale Space and Variational Methods in Computer Vision}, pages 435--446. Springer, 2011.

\bibitem[Rezende et~al.(2014)Rezende, Mohamed, and Wierstra]{rezende2014stochastic}
Danilo~Jimenez Rezende, Shakir Mohamed, and Daan Wierstra.
\newblock Stochastic backpropagation and approximate inference in deep generative models.
\newblock In \emph{International Conference on Machine Learning}, 2014.

\bibitem[Robbins(1956)]{robbins1956empirical}
Herbert Robbins.
\newblock An empirical bayes approach to statistics.
\newblock In \emph{Proceedings of the Third Berkeley Symposium on Mathematical Statistics and Probability}, volume~1, pages 157--163. University of California Press, 1956.

\bibitem[Roberts and Tweedie(1996)]{roberts1996exponential}
Gareth~O. Roberts and Richard~L. Tweedie.
\newblock Exponential convergence of langevin distributions and their discrete approximations.
\newblock \emph{Bernoulli}, 2\penalty0 (4):\penalty0 341--363, 1996.

\bibitem[Ronneberger et~al.(2015)Ronneberger, Fischer, and Brox]{ronneberger2015u}
Olaf Ronneberger, Philipp Fischer, and Thomas Brox.
\newblock U-net: Convolutional networks for biomedical image segmentation.
\newblock In \emph{International Conference on Medical image computing and computer-assisted intervention}, pages 234--241. Springer, 2015.

\bibitem[Salez(2025)]{salez2025markov}
Justin Salez.
\newblock Introduction to markov processes.
\newblock \url{https://www.ceremade.dauphine.fr/~salez/markov.pdf}, December 2025.
\newblock Lecture notes.

\bibitem[Saumard and Wellner(2014)]{saumard2014log}
Adrien Saumard and Jon~A Wellner.
\newblock Log-concavity and strong log-concavity: a review.
\newblock \emph{Statistics surveys}, 8:\penalty0 45, 2014.

\bibitem[Schr{\"o}dinger(1931)]{schrodinger1931umkehrung}
Erwin Schr{\"o}dinger.
\newblock {\"U}ber die umkehrung der naturgesetze.
\newblock \emph{Sitzungsberichte der Preu{\ss}ischen Akademie der Wissenschaften, Physikalisch-mathematische Klasse}, pages 144--153, 1931.

\bibitem[Sohl-Dickstein et~al.(2015)Sohl-Dickstein, Weiss, Maheswaranathan, and Ganguli]{dicksteinthermo}
Jascha Sohl-Dickstein, Eric Weiss, Niru Maheswaranathan, and Surya Ganguli.
\newblock Deep unsupervised learning using nonequilibrium thermodynamics.
\newblock In \emph{International Conference on Machine Learning}, pages 2256--2265. PMLR, 2015.

\bibitem[Song and Ermon(2020)]{song2020improved}
Y.~Song and S.~Ermon.
\newblock Improved techniques for training score-based generative models.
\newblock In \emph{Advances in Neural Information Processing Systems 33: Annual Conference on Neural Information Processing Systems 2020, NeurIPS 2020, December 6-12, 2020, virtual}, 2020.

\bibitem[Song and Ermon(2019)]{song2019generative}
Yang Song and Stefano Ermon.
\newblock Generative modeling by estimating gradients of the data distribution.
\newblock In \emph{Advances in Neural Information Processing Systems (NeurIPS)}, pages 11918--11930, 2019.

\bibitem[Song et~al.(2020)Song, Garg, Shi, and Ermon]{song2020sliced}
Yang Song, Sahaj Garg, Jiaxin Shi, and Stefano Ermon.
\newblock Sliced score matching: A scalable approach to density and score estimation.
\newblock In \emph{Uncertainty in artificial intelligence}, pages 574--584. PMLR, 2020.

\bibitem[Song et~al.(2021)Song, Sohl-Dickstein, Kingma, Kumar, Ermon, and Poole]{song2021scorebased}
Yang Song, Jascha Sohl-Dickstein, Diederik~P. Kingma, Abhishek Kumar, Stefano Ermon, and Ben Poole.
\newblock Score-based generative modeling through stochastic differential equations.
\newblock In \emph{International Conference on Learning Representations}, 2021.

\bibitem[St{\'e}phanovitch et~al.(2025)St{\'e}phanovitch, Aamari, and Levrard]{stephanovitch2025generalization}
Arthur St{\'e}phanovitch, Eddie Aamari, and Cl{\'e}ment Levrard.
\newblock Generalization bounds for score-based generative models: a synthetic proof.
\newblock \emph{arXiv preprint arXiv:2507.04794}, 2025.

\bibitem[Strang(1968)]{strang1968construction}
Gilbert Strang.
\newblock On the construction and comparison of difference schemes.
\newblock \emph{SIAM journal on numerical analysis}, 5\penalty0 (3):\penalty0 506--517, 1968.

\bibitem[Strasman et~al.(2025{\natexlab{a}})Strasman, Ocello, Boyer, Corff, and Lemaire]{strasman2025an}
Stanislas Strasman, Antonio Ocello, Claire Boyer, Sylvain~Le Corff, and Vincent Lemaire.
\newblock An analysis of the noise schedule for score-based generative models.
\newblock \emph{Transactions on Machine Learning Research}, 2025{\natexlab{a}}.
\newblock ISSN 2835-8856.

\bibitem[Strasman et~al.(2025{\natexlab{b}})Strasman, Surendran, Boyer, Le~Corff, Lemaire, and Ocello]{strasman2025wasserstein}
Stanislas Strasman, Sobihan Surendran, Claire Boyer, Sylvain Le~Corff, Vincent Lemaire, and Antonio Ocello.
\newblock Wasserstein convergence of critically damped langevin diffusions.
\newblock In \emph{The Thirty-ninth Annual Conference on Neural Information Processing Systems (NeurIPS)}, 2025{\natexlab{b}}.

\bibitem[Strasman et~al.(2026{\natexlab{a}})Strasman, Cardoso, Corff, Lemaire, and Ocello]{strasman2026forgetting}
Stanislas Strasman, Gabriel Cardoso, Sylvain~Le Corff, Vincent Lemaire, and Antonio Ocello.
\newblock On forgetting and stability of score-based generative models, 2026{\natexlab{a}}.

\bibitem[Strasman et~al.(2026{\natexlab{b}})Strasman, Cardoso, Corff, Lemaire, and Ocello]{strasman2026smc}
Stanislas Strasman, Gabriel~Victorino Cardoso, Sylvain~Le Corff, Vincent Lemaire, and Antonio Ocello.
\newblock Theoretical guarantees for smc-guided diffusion sampling, 2026{\natexlab{b}}.

\bibitem[Strasman et~al.(2026{\natexlab{c}})Strasman, Surendran, and Corff]{strasman2026sgd}
Stanislas Strasman, Sobihan Surendran, and Sylvain~Le Corff.
\newblock Non-asymptotic convergence of stochastic gradient descent in score-based generative models, 2026{\natexlab{c}}.

\bibitem[Tuckerman(2023)]{tuckerman2023statistical}
Mark~E Tuckerman.
\newblock \emph{Statistical mechanics: theory and molecular simulation}.
\newblock Oxford university press, 2023.

\bibitem[Vempala and Wibisono(2019)]{VempalaWibisono2019}
Santosh~S. Vempala and Andre Wibisono.
\newblock Rapid convergence of the unadjusted langevin algorithm: Isoperimetry suffices.
\newblock In \emph{Advances in Neural Information Processing Systems}, volume~32, 2019.

\bibitem[Villani(2009)]{villani_optimal_2009}
C{\'e}dric Villani.
\newblock \emph{Optimal Transport: Old and New}, volume 338 of \emph{Grundlehren der mathematischen Wissenschaften}.
\newblock Springer Berlin Heidelberg, Berlin, Heidelberg, 2009.
\newblock ISBN 978-3-540-71049-3, 978-3-540-71050-9.

\bibitem[Vincent(2011)]{Vincent}
Pascal Vincent.
\newblock A connection between score matching and denoising autoencoders.
\newblock \emph{Neural Computation}, 23\penalty0 (7):\penalty0 1661--1674, 2011.

\bibitem[Yakovlev and Puchkin(2025)]{yakovlev2025generalization}
Konstantin Yakovlev and Nikita Puchkin.
\newblock Generalization error bound for denoising score matching under relaxed manifold assumption.
\newblock In \emph{Conference on Learning Theory}, pages 5824--5891. PMLR, 2025.

\bibitem[Zhang et~al.(2024)Zhang, Yin, Liang, and Liu]{zhang24bv}
Kaihong Zhang, Heqi Yin, Feng Liang, and Jingbo Liu.
\newblock Minimax optimality of score-based diffusion models: Beyond the density lower bound assumptions.
\newblock In Ruslan Salakhutdinov, Zico Kolter, Katherine Heller, Adrian Weller, Nuria Oliver, Jonathan Scarlett, and Felix Berkenkamp, editors, \emph{Proceedings of the 41st International Conference on Machine Learning}, volume 235 of \emph{Proceedings of Machine Learning Research}, pages 60134--60178. PMLR, 21--27 Jul 2024.

\bibitem[Zhang et~al.(2026)Zhang, Xu, Zhou, Fazel, and Du]{zhang2025convergence}
Yiran Zhang, Weihang Xu, Mo~Zhou, Maryam Fazel, and Simon~Shaolei Du.
\newblock Convergence dynamics of over-parameterized score matching for a single gaussian.
\newblock In \emph{International Conference on Learning Representations}, 2026.

\end{thebibliography}

\end{document}